\documentclass{article} 
\usepackage{iclr2027_conference,times}

\usepackage{url}

\usepackage[utf8]{inputenc} 
\usepackage[T1]{fontenc}    
\usepackage[backref=page]{hyperref}       
\usepackage{url}            
\usepackage{booktabs}       
\usepackage{threeparttable}
\usepackage{amsfonts}       
\usepackage{nicefrac}       
\usepackage{microtype}      
\usepackage{xcolor}         

\usepackage[utf8]{inputenc} 
\usepackage[T1]{fontenc}    
\usepackage{hyperref}       
\usepackage{url}            
\usepackage{booktabs}       
\usepackage{enumitem}
\usepackage{mathtools}
\usepackage{todonotes}
\usepackage{amssymb,fge}
\usepackage{thm-restate}
\usepackage{wrapfig}
\usepackage{bbm}
\usepackage{mathrsfs}
\usepackage{soul}
\usepackage{array}
\usepackage{multirow}
\usepackage{algorithm}
\usepackage{tikz}
\usepackage{subcaption}
\usepackage{wrapfig}
\usetikzlibrary{positioning,arrows.meta,shapes.geometric}
\usepackage{nicefrac}       
\usepackage{microtype}      
\usepackage{xcolor}         
\usepackage{colortbl}
\usepackage{enumitem}
\usepackage{xspace}
\usepackage{lineno}

\PassOptionsToPackage{final,pagebackref=true}{hyperref}
\usepackage{hyperref} 
\usepackage{wasysym}
\usepackage{acronym}
\usepackage{xfrac}

\usepackage{amsmath, amsthm, mathtools}
\newtheorem{assumption}{Assumption}
\usepackage[nameinlink]{cleveref}
\usepackage[normalem]{ulem}
\usepackage{changepage}
\usepackage{float}
\usepackage{makecell}
\usepackage{pifont}
\usepackage{tikz}
\usepackage{empheq}
\usepackage{booktabs}

\usepackage{svg}
\usepackage{centernot}

\Crefname{assumption}{Assumption}{Assumptions}

\newcommand{\cmark}{\ding{51}} 
\newcommand{\xmark}{\ding{55}} 
\definecolor{sectiongray}{gray}{0.90}
\definecolor{refgray}{gray}{0.55}
\usepackage{algorithm}
\usepackage{algorithmic}

\hypersetup{
	colorlinks=true,       
	linkcolor=blue,        
	citecolor=violet!80!black,        
	filecolor=magenta,     
	urlcolor=blue
}

\usepackage[most]{tcolorbox}
\newtcolorbox{wideeqbox}{
  width=\linewidth,
  colframe=blue!40!black!10,
  colback=green!05,
  boxrule=\fboxrule,
  arc=0pt,
  left=6pt,
  right=6pt,
  top=0pt,
  bottom=0pt,
  before upper={
    \setlength{\abovedisplayskip}{2pt}%
    \setlength{\belowdisplayskip}{2pt}%
    \setlength{\abovedisplayshortskip}{2pt}%
    \setlength{\belowdisplayshortskip}{2pt}%
  },
}

\graphicspath{{figures/}}

\usepackage{amsmath, amsthm, amsfonts, bm}
\usepackage{mathtools}
\usepackage{thmtools}
\usepackage{thm-restate}
\usepackage{xcolor}
\usepackage{xspace}
\usepackage[framemethod=TikZ]{mdframed}

\newtheorem{theorem}{Theorem}
\newtheorem{proposition}{Proposition}

\def\eqref#1{equation~\ref{#1}}

\def\1{\bm{1}}

\DeclareMathAlphabet{\mathsfit}{\encodingdefault}{\sfdefault}{m}{sl}
\SetMathAlphabet{\mathsfit}{bold}{\encodingdefault}{\sfdefault}{bx}{n}

\def\gN{{\mathcal{N}}}

\newcommand{\pdata}{p_{\rm{data}}}

\newcommand{\E}{\mathbb{E}}

\newcommand{\xhdr}[1]{{\noindent\bfseries #1}.}
\newcommand{\cut}[1]{}

\newcommand{\diff}{\mathrm{d}}

\usepackage[framemethod=TikZ]{mdframed}
\mdfdefinestyle{MyFrame}{%
    linecolor=black,
    outerlinewidth=.3pt,
    roundcorner=5pt,
    innertopmargin=1pt, 
    innerbottommargin=1pt, 
    innerrightmargin=1pt,
    innerleftmargin=1pt,
    backgroundcolor=black!0!white}

\mdfdefinestyle{MyFrame2}{%
    linecolor=white,
    outerlinewidth=1pt,
    roundcorner=1pt,
    innertopmargin=3pt,
    innerbottommargin=2pt,
    innerrightmargin=7pt,
    innerleftmargin=7pt,
    backgroundcolor=black!3!white}

\mdfdefinestyle{MyFrameEq}{%
    linecolor=white,
    outerlinewidth=0pt,
    roundcorner=0pt,
    innertopmargin=0pt,
    innerbottommargin=0pt,
    innerrightmargin=7pt,
    innerleftmargin=7pt,
    backgroundcolor=black!3!white}

\newcommand{\pnot}{p_0}
\newcommand{\Phistar}{\Phi^{\star}}
\newcommand{\vstar}{v^{\star}}
\newcommand{\ustar}{u^{\star}}

\newcommand{\meanflow}{\text{MeanFlow}\xspace}
\newcommand{\stopgradient}{\text{stop gradient}\xspace}

\setcitestyle{round}

\title{On the Stability and Convergence of Flow Map Training}

\author{%
  Ishan Nath\textsuperscript{1}\thanks{Correspondence: \texttt{ishan.nath@stats.ox.ac.uk}},\;
  Tom Stanic\textsuperscript{2,3},\;
  Avishek Joey Bose\textsuperscript{3,4},
  \\
  \textbf{
  \!Gauthier Gidel\textsuperscript{2,3,5},\;
  Quentin Bertrand\textsuperscript{3,6}}
  \\
  \!\textsuperscript{1}Inria, ENS de Lyon, CNRS, Université Claude Bernard Lyon 1, \\ LIP UMR 5668, 69342 Lyon Cedex 07, France \\
  \!\textsuperscript{2}Université de Montréal,
  \textsuperscript{3}Mila - Quebec AI Institute,
  \\
  \!\textsuperscript{4}Imperial College London,
  \textsuperscript{5}CIFAR AI Chair \\
  \!\textsuperscript{6}Université Jean Monnet Saint-Étienne, CNRS, Institut d’Optique Graduate School, \\ Inria, Laboratoire Hubert Curien UMR 5516, F-42023 Saint-Étienne, France \\
}

\iclrfinalcopy 
\begin{document}

\maketitle
\lhead{Preprint. Work in progress.}

\begin{abstract}

Flow maps are a recent method of probabilistic generation that amortize diffusion trajectories into one-step samplers, unlocking efficient generation. While progress has been made in improving their empirical performance, there is still a lack of understanding behind the mechanisms underlying their success. 
Specifically, many flow map learning methods employ stop gradient operators as a crucial part of their training target. 
Adaptive weighting also drastically improves the generation capabilities of flow maps. Although many interpretations have been proposed to support the inclusion of these heuristics, none have analyzed the impact of these assumptions on the training dynamics of flow maps themselves. In this work, we address this issue by providing grounded theory explaining why such tricks are required.
Analyzing the functional gradient flow induced by the MeanFlow and Lagrangian losses, we show their convergence to the true flow map, settling the validity of these objectives. For finite updates, we highlight a severe roadblock to the usage of these losses in practice: the possible unboundedness of the update. 
Adaptive normalization not only provides bounded updates, but, combined with a specific weighting, additionally provides convergence guarantees towards a neighborhood of the solution.
On synthetic data, we show that our functional gradient flow analysis exactly describes the true training dynamics for flow maps.
Experiments on standard image datasets demonstrate that a proper reweighting scheme is essential for strong empirical performance, which is only achieved for a specific range of reweighting strengths.

\end{abstract}

\input{intro_figure.tex}

\section{Introduction}
\label{sec:intro}

\looseness=-1
Modern deep generative models, such as diffusion \citep{sohlDickstein2015diffusion,hoDenoisingDiffusionProbabilistic2020,songScoreBasedGenerativeModeling2021} and flow matching models \citep{lipmanFlowMatchingGenerative2023,albergoBuildingNormalizingFlows2023,liuFlowStraightFast2023}, can now generate realistic data across a wide range of modalities.
Diffusion and flow matching methods are able to generate samples that are indistinguishable from human-generated content, including images \citep{StablediffXL,MidJourney,labs2025flux}, audio \citep{Borsos2023}, and video \citep{Villegas2022,videoworldsimulators2024}.
Diffusion and flow matching have also recently shown great performance on text, either with discrete or continuous approaches \citep{savinov2022step,Gong2022diffusiontext,sahoo2024simple,leeFlowMapLanguage2026,team2026diffusiongemma}.

Once trained, sampling for flow matching models is typically done by numerically solving an ordinary differential equation, \emph{i.e.}, integrating the learned velocity over hundreds of steps, hence requiring hundreds of neural network evaluations.
This high computational cost at inference has motivated the development of few-step generative models, either by straightening the flow \citep{liuFlowStraightFast2023,esser2024scaling}, or via distribution matching approaches \citep{yinOnestepDiffusionDistribution2024,zhouInductiveMomentMatching2025a}.
Current promising one-step generative modeling approaches rely on distilling the inference trajectories and include techniques such as consistency models \citep{song2023consistency}, which learn to directly predict the trajectory endpoint of the flow matching ordinary differential equation from any point along the trajectory.
In the same spirit, flow maps \citep{gengMeanFlowsOnestep2025,gengImprovedMeanFlows2025,boffiFlowMapMatching2025,luOnestepLatentfreeImage2026} learn the solution of the ordinary differential equation between any two times.

\looseness=-1
A variety of approaches for training flow maps have been proposed \citep{gengMeanFlowsOnestep2025,gengImprovedMeanFlows2025,boffiFlowMapMatching2025,zhouTerminalVelocityMatching2026,goldsteinFlowMapLearning2026,fransOneStepDiffusion2025,zhangAlphaFlowUnderstandingImproving2025}, due to the abundance of possible training objectives.
To stabilize training, several tricks are employed, including \stopgradient operators, self-distillation from specialized velocity heads, loss reweighting, gradient normalization, and sampling time pairs from hand-picked distributions \citep{gengMeanFlowsOnestep2025, gengImprovedMeanFlows2025, sabourAlignYourFlow2025, luOnestepLatentfreeImage2026}.
\emph{There is little understanding behind why certain losses work in practice}, and whether the tricks used for training are theoretically sound. In particular, the inclusion of \stopgradient operators has been found to substantially improve flow map training \citep{gengMeanFlowsOnestep2025}; however, it alters gradient directions and has yet to be shown to recover the true flow map.

\xhdr{Contributions} In this paper, we aim to understand the impact of the \emph{\stopgradient} operator and of \emph{adaptive normalization} on the learning dynamics of flow maps, and to show how they provide tangible benefits to the convergence and stability of flow map training.
We show the following:
\begin{enumerate}[topsep=0pt, partopsep=0pt, itemsep=0pt, parsep=0pt, leftmargin=*]
    \item \looseness=-1 \xhdr{Functional convergence of flow maps} We show that, in a specific functional setting, the gradient flow of the \meanflow and Lagrangian objectives converges to the desired flow map (\Cref{thm_meanflow}).
    Importantly, the convergence results rely on the careful \stopgradient placement in the \meanflow and Lagrangian losses.
    \item \looseness=-1 \xhdr{Adaptive normalization} We show that a fundamental limitation of gradient descent is the unboundedness of the update (\Cref{prop_counter_example}), which may amplify arbitrarily small errors.
    We demonstrate that, crucially, adaptive normalization of the loss, an empirically effective stabilization tool \citep{gengMeanFlowsOnestep2025}, counters this blow-up by bounding the update, and provides convergence guarantees under a specific weighting  (\Cref{prop:normalized_convergence}).
    \item \looseness=-1 \xhdr{Empirical validation}
    In synthetic experiments where the true flow map can be calculated, we show that at intermediate training times, the learned maps obtained through training under the \meanflow and Lagrangian objectives equal those predicted by our theory.
    Further experiments on CIFAR-10 and FFHQ 64$\times$64 show that reweighting using adaptive normalization with weights that guarantee stable updates results in learned flow maps with lower FID.
\end{enumerate}
%
\xhdr{Organization}
\looseness=-1
\Cref{sec_reminders} provides reminders on flow matching and flow map training objectives.
\Cref{sec_analysis} includes our main results: we analyze the gradient flow and gradient descent dynamics of flow map training from a functional
perspective.
\Cref{sec_related_work} provides a summary of related work, and \Cref{sec_experiments} provides experiments on synthetic and image datasets that confirm the theory developed in \Cref{sec_analysis}.





\section{Background and preliminaries}
\label{sec_reminders}

\looseness=-1
\xhdr{Flow matching}
Let $\pnot$ be the source distribution, typically a Gaussian $\gN(0,I)$, and $\pdata$ the target data distribution.
The goal of flow matching~\citep{albergoBuildingNormalizingFlows2023,liuFlowStraightFast2023,lipmanFlowMatchingGenerative2023} is to find a velocity field $v: \mathbb{R}^d \times [0, 1] \to \mathbb{R}^d$, such that, if we solve the following ordinary differential equation on $[0, 1]$:
\begin{equation}\label{eq:ode}
    \begin{cases}
      \dot x(t) = v(x(t), t), \\
      x(0) = x_0 \in \mathbb{R}^d,
    \end{cases}
\end{equation}
\looseness=-1
then $x(1) \sim \pdata$ for $x_0\sim p_0$. We denote by $p_s$ the distribution of $x(s)$ if $x_0 \sim p_0$; in particular, $p_1 = \pdata$.
If we know the velocity field $v$, new samples can then be generated by sampling $x_0$ from $p_0$, solving the ordinary differential equation, and using $x(1)$ as the generated point. Conditional flow matching learns a velocity field by minimizing the following loss:
\begin{align}\label{eq_cfm}
    \mathcal{L}^{\mathrm{CFM}}(\theta)
    =
    \mathbb E_{
        x_0 \sim \pnot,
         x_1 \sim \pdata,
        t \sim \mathrm{U}(0, 1)
         }\|v_\theta((1-t)x_0 + t x_1, t) - (x_1 - x_0) \|^2.
\end{align}


\looseness=-1
Numerical integration schemes typically rely on a discretization of the time interval. Achieving accurate samples generally requires a fine discretization, often corresponding to hundreds of function evaluations of the neural network. This can lead to prohibitive computational costs, particularly for high-dimensional data such as high-resolution images.
Several approaches have therefore been proposed to reduce the computational cost of sampling \citep{liuFlowStraightFast2023,pooladian2023multisample,tong2024improving}. Below, we briefly review flow maps as an alternative approach to improving sampling efficiency.

\looseness=-1
\xhdr{Flow maps}
Instead of learning the velocity field $v$, one-step techniques aim to directly learn the \emph{flow}, \emph{i.e.}, the solution of the ordinary differential equation~(\Cref{eq:ode}).
Flow map techniques learn the flow $\Phi(x_s, s, t)$, which maps any starting point $x_s$ at time $s$ to its corresponding solution at a later time $t > s$.
More precisely, for a given velocity $v$, the flow map $\Phi^\star(x_s, s, t)$ is defined as the solution at time $t$ of the following ordinary differential equation, initialized at $x_s$ at time $s$:
\begin{wideeqbox}
    \begin{align}
  \text{$\Phistar(x_s, s, t)$ is the solution of the following ODE}
  \quad
  \quad
    \begin{cases}
      \dot x(t) = v(x(t), t) \\
      x(s) = x_s \in \mathbb{R}^d
    \end{cases}
    \text{at time $t$.}
    \label{eq:flow_map_def}
    \end{align}
\end{wideeqbox}
\looseness=-1
The velocity $v$ can be either a pre-trained network $v_{\theta}$ or the exact closed-form $\vstar$ of~\Cref{eq_cfm} \citep{bertrand2026closed}. Once such a flow map is learned, one can sample from $\pdata$ by sampling $x_0 \sim p_0$ and then computing $\Phi(x_0, 0, 1)$, requiring only a single neural network evaluation.

\looseness=-1
\xhdr{Learning flow maps}
One way to learn flow maps is to find a fixed-point equation that is satisfied by the optimal flow map defined in~\Cref{eq:flow_map_def}.
This can be done by differentiating the definition of the flow map~(\Cref{eq:flow_map_def}) with respect to $s$: this yields the \emph{Eulerian} identity  (\citealt{boffiFlowMapMatching2025}, Prop. 3.6)
  \begin{align} \label{eq:eulerian_id}
    \partial_s \Phistar(x, s, t)
    + \nabla_x \Phistar(x, s, t)
    v(x, s) = 0
    \enspace.
  \end{align}
\looseness=-1
Associated to each flow map $\Phi$ is the average velocity $u$ satisfying $\Phi(x, s, t) = x + (t - s) u(x, s, t)$. The parametrizations are equivalent in theory, but in practice the average velocity parametrization ensures $\Phi(x, s, s) = x$ while learning. 
The popular \emph{\meanflow} objective \citep{gengMeanFlowsOnestep2025} can be seen as a stop gradient version of the average velocity parametrization substituted into the Eulerian identity (\Cref{eq:eulerian_id}):
\begin{wideeqbox}
\begin{align}
      \mathcal{L}^{\mathrm{MF}}(u)
   & =
  \frac 12
  \E_{x, s, t}
  \left\| u(x, s, t) - \mathrm{sg} (
  u_{\mathrm{tgt}}^{\mathrm{MF}}
  (x, s, t))
  \right\|^2
  \enspace,
  \nonumber
  \\
  \label{eq:mf_loss}
  u_{\mathrm{tgt}}^{\mathrm{MF}}(x, s, t)
  & =  v(x, s) + (t - s) \left( \nabla_x u(x, s, t) v(x, s) + \frac{\partial}{\partial s} u(x, s, t) \right)
  \quad \!\!\!
  \text{(\meanflow)}
  \enspace
  .
\end{align}
\end{wideeqbox}
\looseness=-1
Here, $\mathrm{sg}$ denotes the \emph{\stopgradient} operator: if the gradient of the loss $\mathcal L^{\mathrm{MF}}$ is taken with respect to the average velocity $u$ or parameters that control $u$, the terms within the $\mathrm{sg}$ function have their gradient artificially set to 0.

Differentiating the definition of the flow map~(\Cref{eq:flow_map_def}) with respect to $t$ instead yields the \emph{Lagrangian} identity  (\citealt{boffiFlowMapMatching2025}, Prop. 3.5)
  \begin{align}\label{eq:lagrangian_id}
    \partial_t \Phistar(x, s, t) = v(\Phistar(x, s, t), t)
    \enspace.
  \end{align}
The same parametrization $\Phi(x, s, t) = x + (t - s) u(x, s, t)$ with the Lagrangian identity leads to the \emph{Lagrangian} objective. We show in \Cref{sub:ayf_equiv} that this recovers the AYF-LMD objective from \citet{sabourAlignYourFlow2025}, up to reweighting:
\begin{wideeqbox}
    \begin{align}
  \mathcal{L}^{\mathrm{L}}(u)
   & =
  \frac 12
  \E_{x, s, t}
  \left\| u(x, s, t) - \mathrm{sg} (
  u_{\mathrm{tgt}}^{\mathrm{L}}
  (x, s, t))
  \right\|^2
  \enspace,
  \nonumber
  \\
 u_{\mathrm{tgt}}^{\mathrm{L}}
 (x,s,t)
 &=
  v(x + (t - s) u(x, s, t), t) - (t - s) \frac{\partial u}{\partial t} (x, s, t)
  \quad \!\!\!
  \text{(Lagrangian)}
  \enspace
  .
  \label{eq:lagrangian_loss}
    \end{align}
\end{wideeqbox}

%






\looseness=-1
\xhdr{Comments on~\Cref{eq:mf_loss,eq:lagrangian_loss}}
Such \stopgradient-based losses can be motivated in several ways.
Computationally, since all partial derivatives of the network are contained within the \stopgradient, this removes the need for double backpropagation.
Theoretically, a minimizer of $\mathcal L$ is necessarily a flow map, since $\mathcal L \geq 0$ with equality if and only if the Eulerian or Lagrangian identity holds.
However, although a zero training loss implies being a flow map, \emph{the \stopgradient operator might prevent optimization from necessarily decreasing the loss $\mathcal L$}: in practice, the \meanflow loss has been shown not to converge to zero over iterations~\citep{gengImprovedMeanFlows2025}.
Furthermore, \cite{goldsteinFlowMapLearning2026} theoretically showed that the model update induced by the \meanflow loss cannot be the gradient of any scalar loss.

In conjunction with the concurrent work of \citet{shen2026ilearnedstopworrying}, thoroughly discussed in \Cref{sec_related_work}, this work provides the first  (functional) convergence result of the flow map training dynamics for the \meanflow and Lagrangian objectives.

\section{Theoretical analysis of functional flow maps}
\label{sec_analysis}

\looseness=-1
We now present our main theoretical results and rigorously analyze the training dynamics of the \meanflow and the Lagrangian methods.
The main tools required to perform functional gradient descent are introduced in \Cref{sub_setting}.
\Cref{sub:mf_conv} shows our main convergence results: the \meanflow and Lagrangian \emph{gradient flow} training dynamics converge towards the true flow map. This uncovers the non-obvious result: \emph{regressing against the stop gradient target not only yields practical benefits, but provably recovers the true flow map}.
In \Cref{sub:disc_failure}, we analyze how, under \emph{gradient descent}, \meanflow and Lagrangian training can fail to recover the true flow map due to the divergence of iterates,
and how a specific range of adaptive normalization strengths prevents such a failure. 
\subsection{Setting}
\label{sub_setting}
\xhdr{Specific inner product choice}
We consider losses of the form of~\Cref{eq:mf_loss} and~\Cref{eq:lagrangian_loss}. The choice of how to sample $x, s$ and $t$ induces a density $\rho$ on $\mathbb R^{d+2}$.
For example, choosing $(s, t)$ uniformly from the triangle $\Delta = \{0 \leq s \leq t \leq 1\}$ and sampling $x$ from the interpolant distribution at time $s$ yields $\rho(x, s, t) = 2p_s(x)$. With this notation, we can write the losses as
\begin{align}
    \label{eq:general_loss}
    \mathcal L(u) = \int_{x,s, t} \frac 12 \left\| u(x, s, t) - \mathrm{sg}(u_{\mathrm{tgt}}(x, s, t))\right\|^2 \rho(x, s, t) \diff x \diff s \diff t
    \enspace,
\end{align}
with $u_{\mathrm{tgt}}$ either the MeanFlow target (\Cref{eq:mf_loss}) or the Lagrangian target (\Cref{eq:lagrangian_loss}).
As the input space of $\mathcal{L}$ is infinite-dimensional, computing the gradient of $\mathcal{L}$ requires the choice of an inner product on the functional space.
We use the notion of Fréchet derivatives (see \Cref{app_sub_frechet}) with the $L^2(\mathbb R^d \times \Delta)$ inner product. 
%
%
Under this choice of inner product, one can compute the gradient of the loss in~\Cref{eq:general_loss} to be $\nabla_u \mathcal L(u) = \rho(u - u_{\mathrm{tgt}})$.

\xhdr{Exact gradient (no stochasticity)}
\looseness=-1
For most machine learning applications, computing the true deterministic gradient is too expensive, and one has to resort to stochastic gradient algorithms.
To simplify our analysis, we disregard the potential effects of stochasticity on the optimization dynamics, and assume access to the true gradients $\nabla_u \mathcal L^{\mathrm{MF}}$ and $\nabla_u \mathcal L^{\mathrm L}$.

\looseness=-1
\xhdr{Gradient flow} 
For a loss $\mathcal L$ and a fixed step size $\eta >0$, the gradient descent updates $(u_k)$ write
\begin{align}
    \tfrac1\eta(u_{k+1} - u_k) = - \nabla_u \mathcal L(u_k) \enspace.
\end{align}
\looseness=-1
Gradient flow is the continuous-time limit of gradient descent as the step size $\eta \to 0$, where instead of taking discrete steps, 
we continuously follow the negative gradient $- \nabla_u \mathcal L$. 
In terms of the \emph{optimization time $\tau$}, the continuous analogue of the number of steps ($\tau = \eta k$), 
the gradient flow dynamics satisfy
\begin{align}
    \label{eq_general_gradient_flow}
    \partial_\tau u_\tau = - \nabla_u \mathcal L(u_\tau) \enspace.
\end{align}
\looseness=-1
The gradient flow setting provides tractable analysis by avoiding finite step size issues, while remaining a relevant model for practical optimization.

\subsection{Convergence results of the gradient flow dynamics}
\label{sub:mf_conv}
We first show that the flow-map training dynamics $\Phi_\tau$, induced by gradient flow (\Cref{eq_general_gradient_flow}) with the \meanflow and Lagrangian objectives, converge to $\Phistar$.
Equipped with our functional inner product, the functional gradients of the objectives in~\Cref{eq:mf_loss,eq:lagrangian_loss} write
\begin{align}
    \nabla_u \mathcal{L}^{\mathrm{MF}}(u)(x, s, t) 
    &=
    \rho(x, s, t) \left( u (x, s, t) - v(x, s) - (t - s) \left( (\nabla_x u) v(x, s) + \frac{\partial u}{\partial s}  \right) \right)
    \enspace,
  \label{eq:mf_loss_grad}
  \\
  \nabla_u \mathcal L^{\mathrm{L}}(u)(x, s, t) 
  &= \rho(x, s, t) \left( u(x, s, t) - v(x + (t - s) u(x, s, t), t) + (t - s) \frac{\partial u}{\partial t} \right)
  \enspace. 
  \label{eq:lagrangian_gradient}
\end{align}
\looseness=-1
We now solve the gradient flow (\Cref{eq_general_gradient_flow}) that arises from these values of $\nabla_u \mathcal L$. 
We equip $u$ with another variable $\tau$, where \emph{$u_\tau$ denotes the trained average velocity at optimization time $\tau$}. 
Our main result \Cref{thm_meanflow} requires 
standard regularity assumptions on the velocity $v$ and the sampling density $\rho$.
\begin{assumption}[Regularity of the velocity $v$ and the sampling density $\rho$]
    \label{ass_regularity_vel}
    We assume that the velocity $v(x, s)$ is locally Lipschitz in both $x$ and $s$.
    We assume that the sampling density $\rho(x, s, t)$ is locally Lipschitz in $x$, $s$ and $t$ for all $(x, s, t) \in \mathbb R^d \times \Delta$, and $\rho(x, s, t) > 0$ everywhere. 
\end{assumption}
We also assume some regularity on the initial average velocity $u_0$ for the gradient flow dynamics.
\begin{assumption}[Regularity of the average velocity initialization $u_0$]
    \label{ass_regularity_mf}
    The average velocity initialization $u_0$ is $C^1$ on the interior of $\mathbb R^d \times \Delta$, and continuous at the boundary.
\end{assumption}

Under these assumptions, we obtain one of our main results: flow map pointwise convergence.
\begin{theorem}
  \label{thm_meanflow}
  Assume the velocity and the sampling density satisfy \Cref{ass_regularity_vel}, and suppose the initial average velocity $u_0$ satisfies \Cref{ass_regularity_mf}. 
  If the average velocity $u_\tau$ satisfies the functional gradient flow \Cref{eq_general_gradient_flow}, then its associated flow map $\Phi_\tau(x, s, t) := x + (t-s) u_\tau(x, s, t)$ has a closed-form solution for the \meanflow and Lagrangian objectives 
    \begin{wideeqbox}
    \begin{align}
      \Phi_\tau^{\mathrm{MF}}(x, s, t)
      &=
      \Phi_0(\Phistar(x, s, t_\tau), t_\tau, t)  
      \enspace,
      &
      \text{ if }  \mathcal L &= \mathcal L^{\mathrm{MF}} 
      \quad \text{ (MeanFlow)} 
      \enspace,
      \label{eq_cf_mf}
      \\
        \Phi_\tau^{\mathrm{L}}(x, s, t) &= \Phi^\star(\Phi_0(x, s, s_\tau), s_\tau, t) 
      \enspace,
        & 
      \text{ if } \mathcal L &= \mathcal L^{\mathrm{L}}
       \quad  \text{ (Lagrangian)}
      \enspace,
      \label{eq_cf_lag}
  \end{align}
  \end{wideeqbox}
  where $\Phi_0 = x + (t - s) u_0$ and $t_\tau$ (resp. $s_\tau$)
  is an increasing (resp. decreasing) function in $\tau$ with $t_0 = s$, $t_\tau \to t$ as $\tau \to \infty$ (resp. $s_0 = t$, $s_\tau \to s$), whose speed depends on $\rho$.
  Hence, $\Phi_\tau$ converges pointwise to the true flow map as $\tau \to \infty$ under $\mathcal L^{\mathrm{MF}}$. If $\Phistar$ is also continuous in $(x, s)$, the same holds for $\mathcal L^{\mathrm{L}}$.
  
\end{theorem}

\looseness=-1
The proof can be found in~\Cref{sub:pde_converge_proofs}. 
For the \meanflow objective, the key idea is that the PDE induced by the \meanflow target is a \emph{linear} PDE, hence it can be solved by the method of characteristics. Under suitable conditions, the characteristics trace back to the initial conditions, and the speed at which the characteristics are traversed is given by our weighting $\rho$.
For the Lagrangian objective, although the PDE is now non-linear as $u$ appears in the velocity in~\Cref{eq:lagrangian_gradient}, the same method of characteristics can be used to show convergence.

\looseness=-1
The closed form in \Cref{thm_meanflow} also provides insights into how training under the \meanflow loss structurally modifies the learned map: as training progresses, the \meanflow map learns to transport samples using the true flow to increasingly later times $t_\tau$, and the initialization $\Phi_0$ estimates the remaining jump to the time $t$.
On the other hand, under the Lagrangian objective, training teaches the flow map to transport samples from increasingly earlier initial times $s_\tau$, with the initial transport up to time $s_\tau$ given by our initialization $\Phi_0$. 
Both objectives learn the flow map, but by replacing the initialization with the true flow map from opposing directions.
\xhdr{Rates of convergence}
\looseness=-1
Using the convergence results of \Cref{thm_meanflow}, we can derive convergence rates for the learned maps in the $L^{\infty}$ (\Cref{prop_cvg_rate_mf}) and $L^2$-norm (\Cref{prop_rates}).
\begin{restatable}{theorem}{propCvgRateMf}
\label{prop_cvg_rate_mf}
 Assume \Cref{ass_regularity_vel,ass_regularity_mf} hold.
  In addition, assume $\rho(x, s, t) \geq \rho_{\mathrm{min}} > 0$ and $\|u_0\|_{\infty}, \|v\|_\infty \leq M$. Then the MeanFlow closed-form satisfies
    \begin{wideeqbox}
  \begin{align}
    \|\Phi^{\mathrm{MF}}_\tau(x, s, t) - \Phi^\star(x, s, t)\|_{\infty} \leq 2M e^{-\rho_{\mathrm{min}}\tau}
    \enspace.
  \end{align}
  \end{wideeqbox}
  In addition, if $v$ is uniformly $L$-Lipschitz for all $t$, then the Lagrangian closed-form satisfies
    \begin{wideeqbox}
  \begin{align}
    \|\Phi^{\mathrm{L}}_\tau(x, s, t) - \Phi^\star(x, s, t)\|_\infty \leq 2M e^{L} e^{-\rho_{\mathrm{min}}\tau}
    \enspace.
  \end{align}
  \end{wideeqbox}
\end{restatable}

\looseness=-1
The proof of \Cref{prop_cvg_rate_mf} and examples showing these bounds are tight (up to constants) can be found in \Cref{sub:convergence_rates_proof}. The lower bound on $\rho$ is a mild assumption; it holds over compact domains when $\rho$ is positive.

\Cref{prop_rates} formulates convergence rates with respect to the $L^2(p_s/\rho)$ norm (with the abuse of notation that $p_s$ is a function of $x, s, t$): for a function $f : \mathbb R^d \times \Delta \to \mathbb R^d$, its norm $L^2(g)$ with respect to a function $g : \mathbb R^d \times \Delta \to \mathbb R_{\geq 0}$ is given by
  \begin{align}
  \label{eq_l2_norm}
      \|f\|^2_{L^2(g)} = \int_{\mathbb R^d \times \Delta} g(x, s, t) \|f(x, s, t)\|^2 \diff x\diff s \diff t.
  \end{align}
The complementary convergence results (\Cref{prop_rates}) in the $\|\cdot \|_{L^2}$-norm require additional integrability conditions on the displacement of the average velocity solution $u_\tau$ of the gradient flow dynamics \Cref{eq_general_gradient_flow} versus the ground truth average velocity $u^\star(x, s, t) = (\Phi^{\star}(x, s, t) - x)/(t-s)$.
\begin{assumption}[Integrability of the gradient flow average velocities $u_\tau$]
    \label{ass_integrability}
    Suppose $u_\tau - \ustar \in C^1$ on the interior of $\mathbb R^d \times \Delta$ and $p_s\|u_\tau - \ustar\|^2  + p_s(t - s) |(v \cdot \nabla_x  + \partial_s) \|u_\tau - \ustar\|^2|$ is dominated by an integrable function independent of $\tau$ for $\tau \in [0, T]$, for all $T < \infty$. 
\end{assumption}

Under the norm defined in \Cref{eq_l2_norm} and \Cref{ass_integrability}, we show convergence of the average velocity solution of the gradient flow $u_\tau$ to the ground-truth average velocity $u^\star(x, s, t) = (\Phi^{\star}(x, s, t) - x)/(t-s)$. The proof of \Cref{prop_rates} can be found in \Cref{sub:convergence_rates_proof}.
\begin{restatable}{theorem}{propRatesLtwo}
\label{prop_rates}
  Assume \Cref{ass_regularity_vel,ass_regularity_mf,ass_integrability} hold,
   and $\rho(x, s, t) \geq \rho_{\mathrm{min}} > 0$.
  \looseness=-1
  If $\|u_0^{\mathrm{MF}} - \ustar\|^2_{L^2(p_s/\rho)} < \infty$ then the \meanflow trained average velocity $u_{\tau}^{\mathrm{MF}}$ satisfying \Cref{eq_general_gradient_flow} fulfills:
    \begin{wideeqbox}
    \begin{align}
        \|  u_\tau^{\mathrm{MF}} - \ustar \|^2_{L^2(p_s/\rho)} \leq e^{-\rho_{\mathrm{min}}\tau}\|  u_0^{\mathrm{MF}} - \ustar \|^2_{L^2(p_s/\rho)}
        \enspace,
    \end{align}
    \end{wideeqbox}
    hence, the average learned velocity $u_\tau^{\mathrm{MF}}$ converges to $u^{\star}$ with respect to the $L^2(p_s/\rho)$-norm.
\end{restatable}
\looseness=-1
Again the assumption $\|u_0^{\mathrm{MF}} - \ustar\|^2_{L^2(p_s/\rho)} < \infty$ is mild as it is true on compact domains when $\rho > 0$.
As stated in \Cref{sec_reminders}, the \meanflow objective crucially relies on the stop gradient operator, which alters the gradient direction. Specifically, the stop gradient may prevent the dynamics $\Phi_\tau$ from converging toward the true flow map $\Phi^{\star}$.
However, somewhat remarkably, \Cref{prop_rates} shows that the update direction induced by the \meanflow loss implicitly leads us closer to the true flow map anyway: $\Phi_\tau^{\mathrm{MF}} \rightarrow \Phi^{\star}$.


\subsection{Instability and normalization of gradient descent}
\label{sub:disc_failure}

\looseness=-1
In practice, we cannot truly follow the gradient flow (\Cref{eq_general_gradient_flow}). Instead, we may utilize gradient descent-like algorithms: a learning rate $\eta \ll 1$ is chosen, and a gradient step is taken with step size $\eta$
\begin{align}
  u_{k+1} &= u_{k} - \eta \nabla_u \mathcal{L}(u_k)
  \enspace, 
  \label{eq:gradient_descent}
\end{align}
where $\mathcal L = \mathcal L^{\mathrm{MF}}$ or $\mathcal L^{\mathrm{L}}$.
\emph{In contrast to the gradient flow results, we can show that gradient descent is not guaranteed to converge, even for the trivial velocity $v = 0$.}
The idea is that the differential operators present in both $\nabla_u \mathcal L^{\mathrm{MF}}$ and $\nabla_u \mathcal L^{\mathrm{L}}$ are unbounded and can amplify errors.
One can exhibit initial velocities $u_0$ such that gradient descent diverges when applied with the \meanflow and Lagrangian objectives (see \Cref{sub:adaptive_gradient_descent_proof} for the proof of \Cref{prop_counter_example} and details of the counterexample).
\begin{restatable}{proposition}{propCounterExample}\label{prop_counter_example}
    For every $\eta > 0$, $\varepsilon > 0$, and $\rho \in C^\infty$, there exists $u_0$ with initial error $ u_0 - \ustar \in L^2(\rho) \cap C_c^\infty$ and $\|u_0 - \ustar\|_\infty \leq \varepsilon$, but whose iterates under \Cref{eq:gradient_descent} diverge for $\mathcal L = \mathcal L^{\mathrm{MF}}$, or for $\mathcal L = \mathcal L^{\mathrm{L}}$ with $v = 0$.
\end{restatable}


However, if we adapt our step size according to the magnitude of the update, we can avoid a blow-up. 
Since the residual $ u_k - u_{\mathrm{tgt},k}$ can become arbitrarily large, we consider an update of the form
\begin{align}
  u_{k+1} = u_{k} - \eta \rho \frac{r_{k}}{(\|r_k\|_{L^2(\rho)}^2 + c)^p}, 
  \quad
  \text{ with } \quad r_k := u_k - u_{\mathrm{tgt},k} \enspace,
  \label{eq:normalized_update}
\end{align}
for some small fixed $c >0$, and a choice of normalization weighting $p$.
We now show that choosing $p \geq 0.5$ enables control over the update magnitude.
\begin{restatable}{theorem}{propNormalizedConvergence}
\label{prop:normalized_convergence}
  Assume that the velocity $v$, the density $\rho$ and $u_0$ fulfill  \Cref{ass_regularity_vel,ass_regularity_mf}.
  In addition, assume that the iterates of gradient descent satisfy the discrete equivalent of \Cref{ass_integrability}, and  
  that $p_s$ is bounded above: $p_s(x) \leq p_{\mathrm{max}} < \infty$, $\rho$ is bounded below: $\rho(x, s, t) \geq \rho_{\mathrm{min}} > 0$ and $\|u_0 - \ustar\|^2_{L^2(p_s/\rho)} < \infty$. 
  Consider the update scheme given in \Cref{eq:normalized_update}.  
  \begin{itemize}[topsep=0pt, partopsep=0pt, itemsep=0pt, parsep=0pt, leftmargin=*]
      \item For $p \geq 0.5$, the updates have uniformly bounded norms: $\|u_{k+1} - u_k\|^2_{L^2(p_s/\rho)} \leq M < \infty$.
      \item $p=1$ yields convergence for the \meanflow objective to within a neighborhood: if $\| u_k - u_{\mathrm{tgt},k} \|^2_{L^2(\rho)} \leq M(k+1)$ for some $M>0$, i.e. the norm of the residual grows at most linearly, then 
  \end{itemize}
    \begin{wideeqbox}
    \begin{align}
        \limsup_{k \to \infty}\|  u_{k}^{\mathrm{MF}} - \ustar \|^2_{L^2(p_s/\rho)} \leq  \eta p_{\mathrm{max}} (\rho_{\mathrm{min}}^{-1} + \eta c^{-1})
        \enspace.
    \end{align} 
    \end{wideeqbox}
\end{restatable}
\looseness=-1
A proof is found in \Cref{sub:adaptive_gradient_descent_proof}. 
\Cref{prop:normalized_convergence} shows that, for $p=1$, the iterates of the normalized gradient descent (defined in \Cref{eq:normalized_update}) on the \meanflow objective eventually remain in an $\mathcal O(\eta^{1/2})$ neighborhood of the true flow map $\Phi^\star$. 
In \meanflow \citep{gengMeanFlowsOnestep2025}, normalization is done through taking adaptive loss weights $w = 1/(\| u_{k} - u^{\mathrm{MF}}_{\mathrm{tgt}, k} \|^2_2 + c)^p$.
Our result shows that, when the regression error is integrated over space, taking adaptive weights with this integrated error and $p = 1.0$ has iterates that remain in an $\mathcal O(\eta^{1/2})$ ball around the true solution. 
$p = 1.0$ is the value that is found to work the best for from-scratch training on ImageNet $256\!\!\times\!\!256$ \citep{gengMeanFlowsOnestep2025}.

\section{Related work}
\label{sec_related_work}


\xhdr{Trajectory distillation}
One of the most relevant related works is \emph{consistency models} \citep{song2023consistency,song2024improved}, which build a function $f(x_s, s)$ mapping directly to the trajectory endpoint $x_1$. Continuous-time consistency models are trained with a similar loss to the \meanflow loss \citep{luSimplifyingStabilizingScaling2025,sabourAlignYourFlow2025}, except with a fixed $t = 1$, and hence are amenable to the same analysis as we have presented. In practice consistency models also require similar tricks such as tangent normalization for effective training \citep{luSimplifyingStabilizingScaling2025}.

\xhdr{Flow maps}
\looseness=-1
Many works have proposed learning schemes for flow maps. \citet{boffiHowBuildConsistency2025} derive fundamental equations for flow maps and introduce the LSD and ESD objectives. \citet{gengMeanFlowsOnestep2025} introduce the \meanflow objective, which is further improved in \citet{gengImprovedMeanFlows2025} by replacing a stochastic estimate for $v$ with a self-distilled diagonal term. \citet{luOnestepLatentfreeImage2026} reparametrize the objective as $u = (\hat x(x, s, t) - x) / (1 - s)$. $\hat x$ can then be shown to lie in the convex hull of $p_1$ \citep{leeFlowMapLanguage2026}. Align Your Flow \citep{sabourAlignYourFlow2025} introduces flow map distillation methods derived from consistency estimates. These are equivalent to the methods presented in this paper, up to reweighting. Terminal velocity matching \citep{zhouTerminalVelocityMatching2026} uses a modified LSD loss, where the stop gradient teacher velocity $u$ is replaced with an EMA version.
In \citet{goldsteinFlowMapLearning2026}, the authors argue for an alternative loss to \meanflow, due to the presence of spurious minimizers. While this logic is valid for the ESD objective, our results show that for the \stopgradient \meanflow objective, optimization does not stall in a local minimum.

\xhdr{Concurrent work}
\looseness=-1
In very recent work, \citet{shen2026ilearnedstopworrying} derived a regression principle for \stopgradient dynamics. 
Using this, they derive the same closed form for the functional gradient flow for the \meanflow and Lagrangian losses that we have presented in \Cref{thm_meanflow} through a slightly different analysis, and achieve similar $L^\infty$ convergence guarantees.
\citet{shen2026ilearnedstopworrying} encompasses a larger class of problems such as reinforcement learning, but their convergence results are only valid for $\rho=1$.
Here we assume the existence of a density $\rho$ which alters the gradient dynamics, and show convergence in this broader setting.
We also introduce novel $L^2$ convergence rates, and our results on gradient descent are independent of prior work.


\section{Experiments}
\label{sec_experiments}


Experiments on synthetic data (\Cref{fig:equals_closed_form_1d}) validate the faithfulness of the gradient flow closed-form of the flow maps derived in  \Cref{thm_meanflow}.
Experiments on standard image datasets (\Cref{fig_images_ffhq,tab:fid_p_ablation}) confirm the importance of proper adaptive normalization  (\Cref{prop:normalized_convergence}).

\xhdr{Synthetic experiments}
First, we show that learned flow map training dynamics on synthetic data accurately reproduce the closed-form dynamics ($\Phi_\tau^{\mathrm{MF}}$ and $\Phi_\tau^{\mathrm{L}}$) derived in \Cref{thm_meanflow}.
\Cref{thm_meanflow} assumes functional  (\emph{i.e.}, infinite dimensional) optimization:
to recreate the assumed functional setting as accurately as possible, in \Cref{fig:equals_closed_form_1d} we parametrize the average velocity $u$ pointwise on a discrete grid in $\mathbb R \times \Delta$, removing spatial coupling between parameters.
Average velocities are trained with grid function backbones to transport $\mathcal N(0, 1)$ to $\mathcal N(1, 0.36)$ through a linear interpolant path, regressing against the ground truth velocity $v$.
In this setting, \Cref{fig:equals_closed_form_1d} displays the learned average velocities (solid blue and orange lines) and the closed form average velocities, $u_\tau^{\mathrm{MF}}$ and $u_\tau^{\mathrm{L}}$ (black dashed lines), for multiple optimization times $\tau$.
The solid gray line represents the ground-truth average velocity $\ustar$ we try to amortize.
For completeness, exhaustive empirical details can be found in \Cref{sub:validation_appendix}.

\begin{figure}[t]
  \centering
  \includegraphics[width=\textwidth]{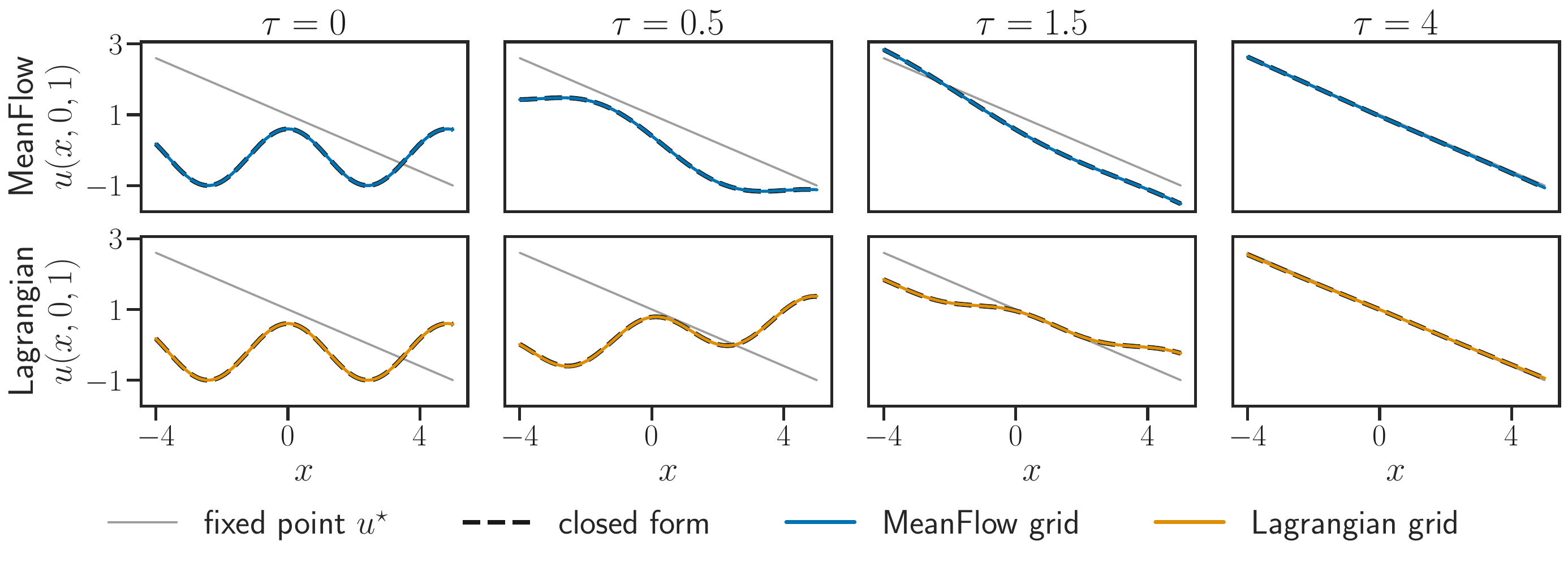}
  \centering
  \caption{\looseness=-1
  \emph{Synthetic Gaussian-to-Gaussian dataset.}
  Snapshots of grid function learned average velocities ($u_k$, solid blue/orange lines) versus their theoretically predicted counterparts $u_\tau$ at four optimization times $\tau$. The two show strong agreement across optimization times $\tau$ and for both the \meanflow (top) and Lagrangian average velocity target (bottom), and both relax towards the ground truth $\ustar$.
  }
  \label{fig:equals_closed_form_1d}
\end{figure}

\begin{figure}[tb]
    \centering
    \includegraphics[width=\linewidth]{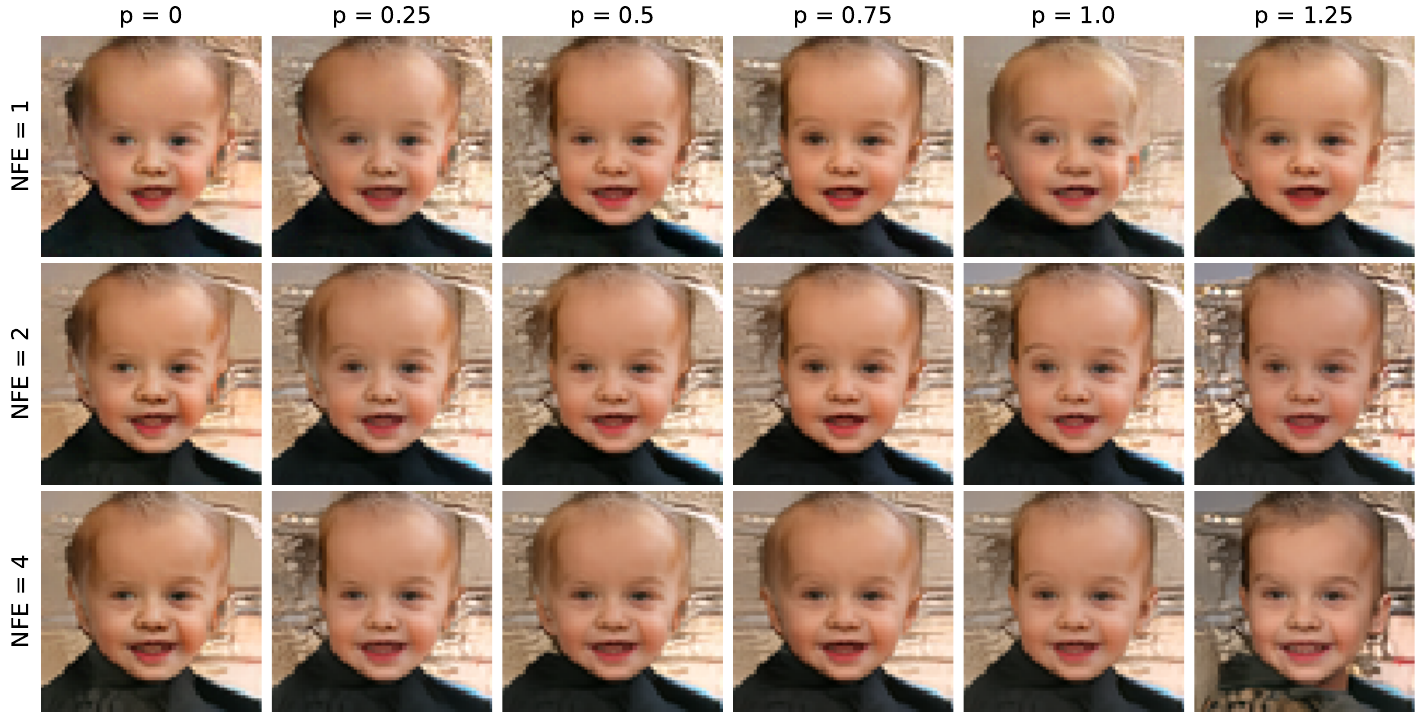}
    \caption{FFHQ dataset. The images are generated from a fixed seed, over different numbers of function evaluations (NFE), with flow maps learned with different normalization exponents $p$.
    }
    \label{fig_images_ffhq}
\end{figure}

\xhdr{Comments on \Cref{fig:equals_closed_form_1d}}
\Cref{fig:equals_closed_form_1d} shows that the learned average velocities match their theorized snapshots $u_\tau^{\mathrm{MF}}$ and $u_\tau^{\mathrm{L}}$ from \Cref{eq_cf_mf,eq_cf_lag}: the black dashed line exactly matches the solid blue/orange lines at each optimization time $\tau$.
Further experiments analyzing the effect of loss weighting and step size on stability are in \Cref{sub:validation_appendix}.

\xhdr{Impact of adaptive gradient normalization}
Our developed theory suggests that flow map training is susceptible to instability due to large magnitude updates, and that adaptive normalization offers a solution. We show that the empirical performance of flow maps distilled from pretrained velocity teachers improves when adaptive normalization is used.
Here we perform experiments on standard image datasets CIFAR-$10$ \citep{Krizhevsky2009} and FFHQ 64$\times$64 \citep{karrasStyleBasedGeneratorArchitecture2019}.
For CIFAR-$10$, we used the pretrained velocity of \citet{tong2024improving}. For FFHQ 64$\times$64, we used the variance exploding denoiser from \citet{EDM} converted into a flow matching model.
The backbone for the CIFAR-10 and FFHQ 64$\times$64 flow maps is an EDM2 U-Net \citep{EDM2Karras} with modified time-embeddings for added stability when computing temporal derivatives \citep{luSimplifyingStabilizingScaling2025}.
We use the \meanflow loss as in \Cref{eq:mf_loss}, weighted by a factor $w = \mathrm{sg}( (\|u_k-u_{\mathrm{tgt},k}\|_2^2 + c)^{-p})$ for various choices of $p$.
Fréchet Inception Distances \citep{heusel2017gans} were computed using torchmetrics \citep{detlefsen2022torchmetrics}.
In \Cref{tab:fid_p_ablation}, we display FIDs of the models with exponential moving average obtained after 200K steps of training, for a variety of values of $p$.
Other experimental details, along with results for the Lagrangian loss on CIFAR-10, are in \Cref{app_expe_large_scale_details}.

\xhdr{Comment on \Cref{tab:fid_p_ablation,fig_images_ffhq}}
\Cref{tab:fid_p_ablation} compares values of the adaptive normalization exponent $p=0$ to $p=1.5$ and multiple numbers of function evaluations (NFE $1$, $2$, $4$), and for multiple high-dimensional datasets.
We observe that $p=0.5$ to $p=1$ gives good performance across datasets.
For $p < 0.5$, the models have noticeably worse FIDs, which corresponds to the region in which the update is possibly unbounded.
For a fixed seed, \Cref{fig_images_ffhq} displays a generated image for multiple values of normalization exponent $p$ and NFE ($1$, $2$, $4$).
Qualitatively $p\leq 0.25$ and $p=1.25$ yield the poorest generated images.

\begin{table}[tb]
\centering
\caption{
FID $\downarrow$ for MeanFlow with different normalization exponents $p$ on CIFAR-10 and on FFHQ 64$\times$64.
Results are reported as mean $\pm$ standard deviation over three runs.}
\label{tab:fid_p_ablation}
\footnotesize
\setlength{\tabcolsep}{4pt}
\renewcommand{\arraystretch}{1.05}
\begin{tabular}{cccccccc}
\toprule
& \multirow{2}{*}{
 \raisebox{-11pt}[0pt][0pt]{\shortstack{Bounded\\[-1pt]updates}}} & \multicolumn{3}{c}{CIFAR-10} & \multicolumn{3}{c}{FFHQ 64$\times$64} \\
\cmidrule(lr){3-5} \cmidrule(lr){6-8}
$p$ &   & NFE 1 & NFE 2 & NFE 4 & NFE 1 & NFE 2 & NFE 4 \\
\midrule
0.00 & \textcolor{red}{\xmark} & $10.62$ {\scriptsize $\pm\,0.21$} & $10.00$ {\scriptsize $\pm\,0.69$} & $9.60$ {\scriptsize $\pm\,0.61$} & $15.38$ {\scriptsize $\pm\,0.84$} & $12.67$ {\scriptsize $\pm\,0.60$} & $12.49$ {\scriptsize $\pm\,0.43$} \\
0.25 & \textcolor{red}{\xmark} & $ 6.81$ {\scriptsize $\pm\,0.23$} & $ 6.70$ {\scriptsize $\pm\,0.30$} & $6.49$ {\scriptsize $\pm\,0.18$} & $10.92$ {\scriptsize $\pm\,0.43$} & $ 9.46$ {\scriptsize $\pm\,0.97$} & $ 9.72$ {\scriptsize $\pm\,0.68$} \\
0.50 & \textcolor{green!60!black}{\cmark} & $ 5.51$ {\scriptsize $\pm\,0.06$} & $ 5.55$ {\scriptsize $\pm\,0.36$} & $4.93$ {\scriptsize $\pm\,0.37$} & $ 8.93$ {\scriptsize $\pm\,0.21$} & $ 7.89$ {\scriptsize $\pm\,0.68$} & $ 8.15$ {\scriptsize $\pm\,0.53$} \\
0.75 & \textcolor{green!60!black}{\cmark} & $\mathbf{5.46}$ {\scriptsize $\pm\,0.06$} & $\mathbf{5.05}$ {\scriptsize $\pm\,0.16$} & $4.78$ {\scriptsize $\pm\,0.04$} & $\mathbf{8.67}$ {\scriptsize $\pm\,0.30$} & $ 6.86$ {\scriptsize $\pm\,0.91$} & $\mathbf{7.69}$ {\scriptsize $\pm\,0.74$} \\
1.00 & \textcolor{green!60!black}{\cmark} & $ 5.61$ {\scriptsize $\pm\,0.20$} & $ 5.17$ {\scriptsize $\pm\,0.05$} & $4.69$ {\scriptsize $\pm\,0.04$} & $10.11$ {\scriptsize $\pm\,0.21$} & $\mathbf{5.64}$ {\scriptsize $\pm\,1.11$} & $ 7.85$ {\scriptsize $\pm\,0.72$} \\
1.25 & \textcolor{green!60!black}{\cmark} & $ 5.93$ {\scriptsize $\pm\,0.19$} & $ 5.25$ {\scriptsize $\pm\,0.08$} & $\mathbf{4.46}$ {\scriptsize $\pm\,0.04$} & $12.70$ {\scriptsize $\pm\,0.04$} & $ 9.48$ {\scriptsize $\pm\,4.69$} & $ 9.89$ {\scriptsize $\pm\,4.17$} \\
1.50 & \textcolor{green!60!black}{\cmark} & $ 6.32$ {\scriptsize $\pm\,0.26$} & $ 5.37$ {\scriptsize $\pm\,0.54$} & $5.48$ {\scriptsize $\pm\,1.53$} & $20.25$ {\scriptsize $\pm\,1.56$} & $57.25$ {\scriptsize $\pm\,35.14$} & $13.92$ {\scriptsize $\pm\,0.46$} \\
\bottomrule
\end{tabular}
\end{table}

\section{Conclusion}
\label{sec:conc}




In this paper, we have analyzed the impact of stop gradient operators and adaptive normalization on flow map training. We have shown how the placement of stop gradients in the \meanflow and Lagrangian losses leads to a gradient update that, when optimized with \emph{gradient flow}, recovers the true flow map pointwise, in $L^\infty$, and in $L^2$ for \meanflow. 
In addition to preventing double backpropagation, \stopgradient operators provably move learned flow maps closer to the ground truth.
However, gradient descent iterations, when applied on the \meanflow and Lagrangian losses, can yield divergence. 
We demonstrate how adaptive normalization gives us both a theoretical and practical solution, enabling higher fidelity flow map learning on CIFAR-10 and FFHQ 64$\times$64 datasets.


\xhdr{Limitations and future directions} \looseness=-1
The scope of the results derived in this paper is limited to the losses discussed in \Cref{eq:mf_loss,eq:lagrangian_loss}. In particular, we do not show convergence for the original ESD and LSD objectives or the Terminal Velocity Matching objective \citep{boffiHowBuildConsistency2025,zhouTerminalVelocityMatching2026}.
Showing convergence for the gradient flow induced by any of these losses is future work.
In addition, our convergence results are derived in a specific functional setting limiting the applicability of the derived results to neural networks, whose parametrization impacts the functional gradient. 
Extending these results to such parametrized functions would require analyzing the PDEs under \emph{preconditioned} gradients. 
The main theory of this work is for the idealized gradient flow setting, where stochasticity-based challenges are ignored, 
as well as other tricks such as training with a fixed percentage of samples taken at the diagonal $s = t$. 
Quantifying the effects of these design choices is left for subsequent studies. 

\subsection*{Acknowledgments}
The authors thank the Blaise Pascal Center for the computational means. It uses the SIDUS \citep{SIDUS} solution. This work was performed using HPC resources from GENCI-IDRIS (Grant 2026-AD011017850).

\subsection*{AI use statement}


In this work, we used generative AI tools for brainstorming, implementing methods, rephrasing sentences, and creating scientific visuals (e.g., helping with TikZ code for plots). 
AI provided us with critical ingredients for proving mathematical claims and assisted with the checking of proofs.

We have not used generative AI tools for generating synthetic data sets, helping develop theoretical models or conceptual frameworks, formulating mathematical claims, assisting in the preliminary write-up of proofs, proposing or refining hypotheses, designing or providing feedback on research methodology or experiments, assisting with translation, cleaning and reformatting datasets, supporting qualitative and thematic data analysis, and interpreting results.

We have reviewed all AI-assisted work. We take responsibility for the final content of this work, including text, claims or artifacts produced with the aid of generative AI.

\bibliographystyle{abbrvnat}
\bibliography{references}

\appendix

\clearpage
\section{Additional Reminders}

\subsection{Fréchet Derivatives}
\label{app_sub_frechet}

Our goal is to minimize the objective in \Cref{eq:general_loss} over the infinite-dimensional function space $u \in \mathcal{U} \subset C^1$.
For such a loss, we can compute its Fréchet derivative with respect to $u$. Since $u_{\mathrm{tgt}}$ is taken under \stopgradient, it is considered constant as a function of $u$.
Hence, we can derive
\begin{align}
    D \mathcal L(u)[h] = \int_{x, s, t} (u(x, s, t) - u_\mathrm{tgt}(x, s, t)) \cdot h(x, s, t) \rho(x, s, t) \diff x \diff s \diff t
    \enspace.
\end{align}
This corresponds to the infinitesimal change in $\mathcal L$ when the input changes infinitesimally in the direction of $h$. The gradient of $\mathcal L$ at $u$ then points in the direction that increases $\mathcal L$ the fastest under an infinitesimal change to the input. Mathematically, this is written as $D \mathcal L(u)[h] = \langle \nabla \mathcal L(u), h \rangle$.
%
Importantly, the gradient \emph{depends on the choice of inner product} on the function space $\mathcal U$.
To formally define the gradient, we fix a choice of inner product, namely the $L^2(\mathbb R^d \times \Delta)$ inner product.

This choice of metric is natural and can be motivated by homogeneity. More precisely, a change in $u$ is felt equally irrespective of its position in space and time. Furthermore, it also does not result in spatial mixing, which is a central ingredient for the success of our analysis. It should be noted that neural network parametrizations do not naturally have this geometry.

Under this inner product, we see that $\nabla \mathcal L(u) = \rho(u - u_{\mathrm{tgt}})$.
%

\section{Main Proofs}
\label{sec:proofs}

\subsection{Proofs of PDE convergence}
\label{sub:pde_converge_proofs}

\begin{theorem}\label{eq:mf_pde_convergence_proof}
  Under the PDE in \Cref{eq_general_gradient_flow}, with $\mathcal{L} = \mathcal{L}^{\mathrm{MF}}$, under \Cref{ass_regularity_vel,ass_regularity_mf}, $u_\tau$ can be written down in closed form:
  \begin{align}
      u_\tau(x, s, t) = u_0(\Phistar(x, s, t_\tau), t_\tau, t) F^{-1}(\tau) + \frac{\Phistar(x, s, t_\tau) - x}{t - s},
  \end{align}
  where $t_\tau = t - (t - s) F^{-1}(\tau)$, and where $F^{-1}(\tau)$ is a decreasing function depending on $\rho$ that quantifies how close we are to $t$. This implies $u_\tau \to \ustar$ pointwise.
\end{theorem}

\begin{proof}
The key idea is that this is a linear PDE, hence it may be solved exactly. We fix $t$ and introduce variables $X \in \mathbb{R}^{d}$ and $S \in [0, 1]$ satisfying
\begin{align*}
  \frac{\diff X}{\diff \tau}(\tau) &= - \rho(X, S, t) (t - S) v(X, S), \\
  \frac{\diff S}{\diff \tau}(\tau) &= - \rho(X, S, t) (t - S).
\end{align*}
Such $S$ and $X$ exist under the conditions given in \Cref{sub_sufficient_condition}. Evaluating the total derivative, we find that
\begin{align*}
  \frac{\diff}{\diff \tau} u_\tau(X, S, t) &= \frac{\partial}{\partial \tau} u_\tau(X, S, t) + \nabla_x u_\tau  \frac{\diff X}{\diff \tau} + \frac{\partial u_\tau}{\partial s} \frac{\diff S}{\diff \tau} \\
  &= - \rho(X, S, t) \left( u_\tau(X, S, t) - v(X, S) \right).
\end{align*}
Introducing an integrating factor $F$ satisfying $\diff F /  \diff \tau = \rho F$ and $F(0) = 1$, we may write this as
\[
\frac{\diff}{\diff \tau} \left( F u \right) = F(\tau) \rho(X, S, t) v(X, S).
\]
Prescribing boundary conditions $X(\mathcal{T}) = x$ and $S(\mathcal{T}) = s$, we find that $S(\tau) = t - (t - s) F(\tau) / F(\mathcal{T})$, and $S(0) = t - (t - s) F^{-1}(\mathcal{T}) := t_{\mathcal{T}}$. Moreover,
\begin{align*}
  \frac{\diff X}{\diff S} = \frac{\diff X}{\diff \tau} \frac{\diff \tau}{\diff S} = v(X, S),
\end{align*}
hence $X$ follows the velocity field along the characteristic curves. This yields $X(\tau) = \Phistar(x, s, S(\tau))$. Solving the ODE now shows that, for $s < t$,
\begin{align*}
  u_{\mathcal T}(x, s, t) &= u_0(\Phistar(x, s, t_{\mathcal{T}}), t_{\mathcal{T}}, t) F^{-1}(\mathcal{T}) + \int_{0}^{\mathcal{T}} \frac{F(\tau) \rho(X(\tau), S(\tau), t)}{F(\mathcal{T})} v(X(\tau), S(\tau)) \diff \tau \\
  &= u_0(\Phistar(x, s, t_{\mathcal{T}}), t_{\mathcal{T}}, t) F^{-1}(\mathcal{T}) + \frac{1}{t - s} \int_{s}^{t_{\mathcal{T}}} v(\Phistar(x, s, S), S) \diff S,
\end{align*}
which gives the closed form solution
\begin{align}
  u_\tau(x, s, t) = u_0 (\Phistar(x, s, t_{\tau}), t_{\tau}, t) F^{-1}(\tau) + \frac{\Phistar(x, s, t_{\tau}) - x}{t - s}. \label{eq:mf_closed_form}
\end{align}
In terms of $\Phi$, we obtain that
\begin{align}
    \Phi_\tau(x, s, t) &= x + (t - s) u_\tau(x, s, t) \\
    &= \Phistar(x, s, t_\tau) + u_0(\Phistar(x, s, t_\tau), t_\tau, t)\underbrace{ (t - s) F^{-1} (\tau)}_{t - t_\tau} \\
    &= \Phi_0(\Phistar(x, s, t_\tau), t_\tau, t).
\end{align}

Under mild conditions on $\rho$, for example if $\rho$ is bounded below on each trajectory $(x_r)_s^t$, which is true since $\rho > 0$ and each trajectory is a compact set, then $F^{-1}(\tau) \to 0$ as $\tau \to \infty$, since $\dot F \geq \rho_{\mathrm{min}}F \implies F(\tau) \geq e^{\rho_{\mathrm{min}}\tau}$. So, if $u_0$ is locally bounded at the diagonal (which is true since $u_0$ is continuous there), the first term disappears and
\begin{align}
  \lim_{\tau \to \infty} u_\tau(x, s, t) = \frac{\Phistar(x, s, t) - x}{t - s} = \ustar(x, s, t).
\end{align}

We have thus shown that ideal training under the \meanflow loss converges to the average velocity. Moreover, we have shown how the speed of convergence is dependent on a choice of weighting $\rho$.
\end{proof}

$\Phi_\tau$ can be shown to satisfy a similar PDE,
\begin{align}
  \frac{\partial}{\partial \tau} \Phi_\tau(x, s, t) = (t - s) \rho(x, s, t) \left( (\nabla_x \Phi_\tau) v(x, s) + \frac{\partial \Phi_\tau}{\partial s} \right),
\end{align}
which can be solved in an analogous manner.

\begin{theorem}\label{eq:lag_pde_convergence_proof}
  Under the PDE in \Cref{eq_general_gradient_flow}, with $\mathcal{L} = \mathcal{L}^{\mathrm{L}}$, under \Cref{ass_regularity_vel,ass_regularity_mf}, $\Phi_\tau$ can be written down in closed form:
  \begin{align}
            \Phi_\tau(x, s, t) = \Phi^\star(\Phi_0(x, s, s_\tau), s_\tau, t),
  \end{align}
  where $s_\tau = s + (t - s) F^{-1}(\tau)$, and where $F^{-1}(\tau)$ is a decreasing function depending on $\rho$ that quantifies how close we are to $s$. In particular if $\Phi^\star$ is continuous in $(x, s)$, then $\Phi_\tau$ converges pointwise to the true flow map as $\tau \to \infty$.
\end{theorem}

\begin{proof}
This is now a non-linear PDE, since $\Phi_\tau$ appears inside the velocity $v$. However, it can still be solved through characteristic methods. We first determine the PDE for $\Phi$ through the chain rule:
\begin{align*}
    \frac{\partial}{\partial \tau} \Phi_\tau(x ,s, t) &= (t - s) \frac{\partial}{\partial \tau} u_\tau(x, s, t) \\
    &= - (t - s) \rho(x, s, t) \left( \frac{\partial \Phi_\tau}{\partial t}(x, s, t) - v(\Phi_\tau(x, s, t), t) \right).
\end{align*}

Defining $\dot T =  (T - s) \rho(x, s, T)$, we have
\begin{align*}
  \frac{\diff}{\diff \tau} \Phi_\tau(x, s, T) = (T - s) \rho(x, s, T) v(\Phi_\tau(x, s, T), T).
\end{align*}
Let $F$ be an integrating factor for $\rho$ with $F(0) = 1$, and assume $F$ is finite. If $T(\mathcal{T}) = t$, we see that $T = s + (t - s) F(\tau) / F(\mathcal{T})$, and so $T(0) = s + (t - s)F^{-1}(\mathcal{T}) := s_{\mathcal{T}}$. Then
\begin{align*}
  \frac{\diff \Phi_\tau}{\diff T} = \frac{\diff \Phi_\tau}{\diff \tau} \frac{\diff \tau}{\diff T} = v(\Phi_\tau, T),
\end{align*}
so $\Phi_\tau$ flows along the velocity field $v$. Since at $\tau = 0$, $T(0) = s_\mathcal{T}$ and $\Phi_\tau(x, s, T) = \Phi_0(x, s, s_{\mathcal{T}})$, we get that
\begin{align}
  \Phi_\tau(x, s, t) = \Phistar(\Phi_0(x, s, s_{\tau}), s_{\tau}, t). \label{eq:lag_closed_form}
\end{align}
If $F^{-1}(\tau) \to 0$, then $s_\tau \to s$, and so $\Phi_0(x, s, s_\tau) \to x$ since $u_0$ is locally bounded. Therefore under the continuity of $\Phistar$, $\Phi_\tau \to \Phistar$ pointwise.
\end{proof}

\subsection{Convergence rates}
\label{sub:convergence_rates_proof}

\propCvgRateMf*

\begin{proof}
  For the \meanflow closed form, we have that
  \begin{align*}
    |\Phi_\tau^{\mathrm{MF}}(x, s, t) - \Phistar(x, s, t)| &= |u_0(\Phistar(x, s, t_\tau), t_\tau, t) (t - s)F^{-1}(\tau) + \Phistar(x, s, t_\tau) - \Phistar(x, s, t)| \\
    &\leq (t - s) F^{-1}(\tau) \left( \|u_0\|_{\infty} + \frac{1}{t - t_\tau} \int_{t_\tau}^{t} \|v(\Phistar(x, s, r), r)\|_\infty \diff r \right) \\
    &\leq 2M(t - s) F^{-1}(\tau).
  \end{align*}
  If $\rho \geq \rho_{\mathrm{min}}$, then $\dot F = F \rho \geq F \rho_{\mathrm{min}}$ and $F(0) = 1$ gives, by Gr\"onwall, $F(\tau) \geq e^{\rho_{\mathrm{min}} \tau}$, and hence $F^{-1}(\tau) \leq e^{-\rho_{\mathrm{min}}\tau}$. This bound is uniform over all $(x, s, t)$: taking a supremum over all $x, s, t$ finishes.
\end{proof}

A bound on the convergence for the Lagrangian solution requires stronger assumptions. The closed-form solution in \Cref{eq:lag_closed_form} shows the Lagrangian-trained flow map is $\Phi^\star(\Phi_0(x, s, s_\tau), s_\tau, t)$. To have this be close to $\Phi^\star(x, s, t) = \Phi^\star(\Phi^\star(x, s, s_\tau), s_\tau, t)$, we need $\Phi_0(x, s, s_\tau)$ close to $\Phi^\star(x, s, s_\tau)$, and a Lipschitz bound on the flow map depending on $\Phi^\star$.

Note that a Lipschitz bound on the flow map follows from a Lipschitz bound on the velocity: if $v$ is uniformly $L$-Lipschitz in $x$ for all $t$, then from Cauchy--Schwarz, we can obtain
\begin{align*}
\frac{\diff}{\diff t} |\Phi^\star(x, s, t) - \Phi^\star(y, s, t)| &\leq |v(\Phi^\star(x, s, t), t) - v(\Phi^\star(y, s, t), t)| \\
&\leq L|\Phi^\star(x, s, t) - \Phi^\star(y, s, t)|,  \\
|\Phi^\star(x, s, s) - \Phi^\star(y, s, s)| &= |x - y|,
\end{align*}
which tells us that the flow map is $e^{L(t - s)}$-Lipschitz for fixed $(s, t)$ by Gr\"onwall.

\begin{proof}
  First, we obtain a bound for the closeness of $\Phi_0(x, s, s_\tau)$ and $\Phistar(x, s, s_\tau)$, in a similar manner to above:
  \begin{align*}
    |\Phi_0(x, s, s_\tau) - \Phistar(x, s, s_\tau)| &\leq |\Phi_0(x, s, s_\tau) - x| + |\Phistar(x, s, s_\tau) - x| \\
    &\leq 2M(t - s) F^{-1}(\tau),
  \end{align*}
  from calculations similar to the above. Then the Lipschitzness of $\Phistar(\cdot, s_\tau, t)$ gives
  \[
  |\Phi_\tau^{\mathrm{L}}(x, s, t) - \Phistar(x, s, t)| \leq 2M (t - s) e^{L(t - s_\tau)} F^{-1}(\tau).
  \]
  Again, if $\rho \geq \rho_{\mathrm{min}}$, we get $F^{-1}(\tau) \leq e^{-\rho_{\mathrm{min}}\tau}$, which proves the claim.
\end{proof}

We can show that the $e^L$ factor cannot be removed. Take $\rho = 1$, and the velocity
\begin{align}
  v(x, t) =
  \begin{cases}
    1 &  x \geq 1/L, \\
    Lx & x \in [-1/L, 1/L], \\
    -1 & x \leq -1/L.
  \end{cases}
\end{align}
The corresponding flow map is $\Phi^\star(x, s, t) = e^{L(t - s)}x$ for $|x| < e^{-L}/L$. Suppose we have initialization $u_0(x, s, t) = 1$. Then the closed form gives that, for sufficiently large $\tau$,
\begin{align*}
  \Phi_\tau^{\mathrm{L}}(0, 0, 1) &= \Phi^\star(\Phi_0(0, 0, e^{-\tau}), e^{-\tau}, 1) = \Phi^\star(e^{-\tau}, e^{-\tau}, 1) \\
  &= e^{L(1 - e^{-\tau})} e^{-\tau}
\end{align*}
whereas $\Phi^\star(0, 0, 1) = 0$. This gives a separation of $e^{L} e^{-\tau}$, which is tight against our upper bound up to a constant. We can verify the \meanflow map is $\Phi_\tau^{\mathrm{MF}}(0, 0, 1) = \Phi_0(0, 1 - e^{-\tau}, 1) = e^{-\tau}$, which avoids the $e^{L}$ factor.

\xhdr{$L^2$ convergence} For the following proof, it will be useful to define the \emph{material derivative} operator $\mathcal D_v = (v \cdot \nabla + \partial_s)$. Then the \meanflow PDE is equivalent to
\begin{align}
    \partial_\tau u_\tau = - \rho( u_\tau - v - (t - s) \mathcal D_v u_\tau).
\end{align}
We then note that, for an arbitrary field $u$ (not necessarily satisfying the \meanflow PDE),
\begin{align*}
    \mathcal D_v u(\Phistar(z, 0, s), s) &= \nabla u(\Phistar(z, 0, s), s) v(\Phistar(z, 0, s), s)  + \partial_s u(\Phistar(z, 0, s), s) \\
    &= \nabla u (\Phistar(z, 0, s), s) \partial_s \Phistar(z, 0, s) + \partial_s u (\Phistar(z, 0, s), s) \\
    &= \frac{\diff}{\diff s} u(\Phistar(z, 0, s), s).
\end{align*}

\propRatesLtwo*

\begin{proof}
  We note that $u_\tau - \ustar$ follows the PDE
  \begin{align*}
    \frac{\partial}{\partial \tau} (u_\tau - \ustar) &= - \rho\left( u_\tau - v - (t - s) \mathcal{D}_v u_\tau \right) \\
    &= - \rho \left( u_\tau - \ustar + (t - s) \mathcal{D}_v \ustar - (t - s) \mathcal{D}_v u_\tau \right) \\
    &= - \rho (I - (t - s)\mathcal{D}_v)(u_\tau - \ustar).
  \end{align*}
  Now, we may differentiate under the integral sign, due to the regularity offered by \Cref{ass_integrability}:
  \begin{align*}
    \frac{\diff}{\diff \tau} &\int \frac{p_s(x)}{\rho(x, s, t)} \|u_\tau - \ustar\|^2 \diff x \diff s \diff t \\
    &= - \int 2p_s(x) (I - (t - s) \mathcal{D}_v)(u_\tau - \ustar) \cdot (u_\tau - \ustar) \diff x \diff s \diff t \\
    &= - \int p_s(x) (2I - (t - s) \mathcal{D}_v) \|u_\tau - \ustar\|^2 \diff x \diff s \diff t \\
    &= - 2\|u_\tau - \ustar\|^2_{L^2(p_s)} + \int p_s(x) (t - s) \mathcal{D}_v \|u_\tau - \ustar\|^2 \diff x \diff s \diff t.
  \end{align*}
  To analyze the last term we apply a change of variables $x = \Phistar(z, 0, s)$, noting that if $x \sim p_s$, then $z \sim p_0$:
  \begin{align*}
    \int p_s(x) &(t - s) \mathcal{D}_v \| u_\tau - \ustar\|^2 \diff x\diff s \diff t \\
    &= \int p_0(z) (t - s) \mathcal{D}_v \|u_\tau(\Phistar(z, 0, s), s, t) - \ustar(\Phistar(z, 0, s), s, t) \|^2 \diff z \diff s \diff t \\
    &= \int p_0(z) (t - s) \frac{\diff}{\diff s} \|u_\tau(\Phistar(z, 0, s), s, t) - \ustar(\Phistar(z, 0, s), s, t)\|^2 \diff s \diff z \diff t \\
    &= \int p_0(z) \left[  (t - s) \| u_\tau(\Phistar(z, 0, s), s, t) - \ustar(\Phistar(z, 0, s), s, t)\|^2 \right]_{s = 0}^{t} \diff z \diff t  \\
    & \qquad + \int p_0(z) \|u_\tau(\Phistar(z, 0, s), s, t) - \ustar(\Phistar(z, 0, s), s, t)\|^2 \diff z \diff s \diff t \\
    &= -\int t p_0(x) \|u_\tau(x, 0, t) - \ustar(x, 0, t)\|^2 \diff x \diff t + \|u_\tau - \ustar\|^2_{L^2(p_s)}.
  \end{align*}
  Combining with the other part gives
  \begin{align*}
    \partial_\tau \|u_\tau - \ustar\|^2_{L^2(p_s/\rho)} &= \frac{\partial}{\partial \tau} \int \frac{p_s(x)}{\rho(x, s, t)} \|u_\tau - \ustar\|^2 \diff x \diff s \diff t \\
    &= - \int p_s(x) \|u_\tau - \ustar\|^2 \diff x \diff s \diff t - \int t p_0(x) \|u_\tau - \ustar\|^2 \diff x \diff t \\
    &\leq - \rho_{\mathrm{min}} \int \frac{p_s(x)}{\rho(x, s, t)} \|u_\tau - \ustar\|^2 \diff x \diff s \diff t = - \rho_{\mathrm{min}} \|u_\tau - \ustar\|^2_{L^2(p_s/\rho)}.
  \end{align*}
  Hence we have that $\|u_\tau - \ustar\|^2_{L^2(p_s/\rho)} \leq e^{-\rho_{\mathrm{min}}\tau} \|u_0 - \ustar\|^2_{L^2(p_s/\rho)}$.
\end{proof}

Although the assumption that $\rho \geq \rho_{\mathrm{min}} > 0$ cannot hold for a density over $\mathbb R^d$, the same inequality holds when integrating over a compact domain $\{ (\Phistar(x, a, s), s, t) \mid x \in K, a \leq s \leq t \leq b\}$. Moreover if $\rho \leq \rho_{\mathrm{max}} < \infty$, then $\|u_\tau - \ustar\|^2_{L^2(p_s/\rho)} \geq \rho_{\mathrm{max}}^{-1} \|u_\tau - u^\star\|^2_{L^2(p_s)}$, hence $u_\tau \to u^\star$ in $L^2(p_s)$.

\subsection{Sufficient conditions for convergence}
\label{sub_sufficient_condition}

We now state sufficient conditions under which the above proofs are valid. The closed-form and $L^\infty$ proofs rely on the local existence and uniqueness of characteristics $X, S$ and $T$, and the application of the chain rule. This follows from the Picard--Lindel\"of theorem under certain assumptions on the velocity $v$ and the density $\rho$:
\begin{itemize}
    \item $v$ is locally Lipschitz in $x, s$.
    \item $\rho$ is locally Lipschitz in $x, s, t$.
\end{itemize}
This gives local existence and uniqueness. We can extend these local characteristics until they either approach a boundary or $X \to \infty$. The latter cannot happen if the true flow $\Phistar$ is well-defined everywhere. We analyze conditions under which the former is avoided below.

The above derivations only work if the quantity $F$ does not diverge in finite time, otherwise the characteristics hit the boundary $s = t$, where the PDE is unspecified. If $\rho$ is upper bounded by $M$, then this holds, since $F$ satisfies the ODE $\dot F = \rho F \leq M F$ and $F(0) = 1$, giving $F(\tau) \leq e^{M \tau} < \infty$.




If we suspend the assumption that $\rho$ must be locally Lipschitz everywhere, the following proposition characterizes the choices of $\rho$ for which the characteristics satisfy $s < t$ for all $\tau$.

\begin{proposition}
  Suppose that $\rho(x, s, t) \asymp g(t - s)$, where $1/g$ is bounded outside of a neighborhood of $0$ and $\rho$ is bounded below on trajectories. Then $F(\tau) < \infty$ for all $\tau$ if and only if
  \begin{align}
    \int_{0}^{1} \frac{1}{r g(r)} \diff r = \infty.
  \end{align}
\end{proposition}

\begin{proof}
  Note that showing $F(\mathcal{T}) < \infty$ is equivalent to showing $S(0) < t$ or $T(0) > s$ for any choice of $\mathcal{T}$, depending on whether we are in the Eulerian or Lagrangian frame. Define a new variable $R(\tau)$ by
  \begin{align}
    R(\tau) =
    \begin{cases}
      t - S(\mathcal{T} - \tau) & \text{under the Eulerian loss}, \\
      T(\mathcal{T} - \tau) - s & \text{under the Lagrangian loss}.
    \end{cases}
  \end{align}
  Then we can note that in both formulations, $\dot R = -\rho(X, S, T) R$ and $R(0) = t - s$. Letting $\tau_0$ be the time until $R$ hits $0$ (which is possibly $\infty$), we obtain that
  \[
  \int_{0}^{t - s} \frac{1}{r \rho(r)} \diff r = \tau_0.
  \]
  Assuming that $\rho(X, S, T) \asymp g(T - S) = g(R)$ gives that
  \[
  \tau_0 \asymp \int_{0}^{t - s} \frac{1}{r g(r)} \diff r,
  \]
  hence $\tau_0$ is finite if and only if the integral is finite, noting that $0 < t - s \leq 1$, and the integral from $t-s$ to $1$ is finite by the boundedness of $g^{-1}$ away from $0$.
\end{proof}

Choosing $g(r) = 1$ ensures convergence while having $\rho$ remain bounded.

\subsection{Gradient descent proofs}
\label{sub:adaptive_gradient_descent_proof}

For the gradient descent iterates $u_k$, let $\Phi_k(x, s, t) = x + (t - s) u_k(x, s, t)$.

\propCounterExample*

\begin{proof}
  We first write down the recurrence in terms of the errors $e_k = u_k - \ustar$, using the fact that $\ustar - v - (t - s) \mathcal D_v \ustar = 0$, and $\ustar - v(\Phistar) + (t-s)\partial_t \ustar = 0$. For the \meanflow loss,
  \begin{align}
      e_{k+1} = e_k - \eta \rho(e_k - (t - s) \mathcal D_ve_k),
  \end{align}
  and for the Lagrangian loss,
  \begin{align}
      e_{k+1} = e_k - \eta \rho[e_k + (t - s) \partial_t e_k - (v(\Phi_k, t) - v(\Phistar, t))].
  \end{align}

  We give an example of a small error blowing up. Consider $\rho = 1$ and
  \begin{align}
    e = e_0 = \varepsilon (t - s)^{m} a
  \end{align}
  for some value $m \in \mathbb N$ with $\eta (m+1) > 2$ and a constant vector $a$ with $\|a\| = 1$. Then $-(t - s) \mathcal{D}_v e = (t - s) \partial_t e = m \varepsilon (t - s)^{m} a = m e$, hence $e_{k+1} = (1 - \eta(m+1))e_k$ by induction, for any velocity $v$ under the \meanflow loss, and for $v = 0$ under the Lagrangian loss. So $e_k = \varepsilon(1 - \eta(m+1))^{k} (t - s)^{m} a$. Since $\eta > 2(m+1)^{-1}$, this diverges as $k \to \infty$.

  To construct $e_0 \in L^2(\rho)$, we can simply multiply by a smooth cutoff function which is equal to one on a neighborhood of a compact set $K$. Since this update is local, on $K$ we will still have $e_k = \varepsilon(1 - \eta (m+1))^k (t - s)^m a$, and hence still obtain divergence in both $L^2(\rho)$ and $L^\infty$.

  We now generalize our construction to smooth $\rho$. We pick some point $(x, s, t)$ on which $\rho > 0$. By continuity, there is a neighborhood around this point on which $\rho$ is positive.

  Within this neighborhood, we choose $V = \{(\Phistar(x, r, s), s, t) \mid x \in K, s \in [s_1, s_2], t\in[t_1, t_2]\}$, with $t_1 > s_2$, for some fixed reference time $r < s_1$. We would like to find an initial error $e_0$ such that
  \begin{align}
      \rho (e_0 - (t - s) \mathcal D_v e_0) = m e_0 \implies \mathcal D_v e_0(x,s,t) = \frac{1-m/\rho}{(t-s)}e_0(x, s, t),
  \end{align}
  at least in some neighborhood. Consider choosing $e_0(x, s, t)$ to be constant for all $(x, s, t) \in \{(\Phistar(x, r, s_1), s_1, t) \mid x \in K, t \in [t_1, t_2]\}$. Then, we may define $e_0$ implicitly on the rest of $V$ since
  \begin{align}
      \mathcal D_v e_0(\Phistar(x, r, s), s, t) = \frac{\diff}{\diff s} e_0(\Phistar(x, r, s), s, t).
  \end{align}
  It can also be defined explicitly through
  \begin{align}
      e_0(\Phistar(x, r, s), s, t) = C_\varepsilon\exp \left( \int_{s_1}^s \frac{1 - m/\rho(\Phistar(x, r, s'), s', t)}{t - s'} \diff s'\right) a,
  \end{align}
  where $a$ is a constant vector with norm 1, and $C_\varepsilon$ is carefully chosen to bound the $L^\infty$ norm of $e_0$ on $V$ by $\varepsilon$. This suffices as an example of $e_0$ satisfying $\rho(e_0 - (t - s) \mathcal D_v e_0) = m e_0$, at least within $V$. Hence we obtain that $e_{k+1} = (1 - \eta m) e_{k}$, and so if $\eta > \frac 2m$, then this diverges within $V$.

  Again, we may make this well-defined by multiplying by a cutoff function that is one on a neighborhood of a compact subset $K$ of $V$, but now we must also assume this cutoff is $0$ on the exterior of some open subset of $V$, so that the local $e_0$ counterexample above extends to a globally defined counterexample. Moreover $e_0$ must be non-zero somewhere, so we obtain that its iterates must diverge in both $L^\infty$ and $L^2$ norms.

  A similar proof works for the Lagrangian loss with $v = 0$, which we do not detail here.
\end{proof}

This counterexample for the \meanflow loss is independent of the flow map we are trying to learn. The phenomenon that increasingly small $\eta$ is required to fully learn all modes of the solution is due to the inherent \emph{stiffness} of the \meanflow and Lagrangian PDEs.

\propNormalizedConvergence*

\begin{proposition}
  Assume \Cref{ass_regularity_vel,ass_regularity_mf} and the discrete equivalent of \Cref{ass_integrability} hold. Consider the update scheme given in \Cref{eq:normalized_update}. Assume that $p_s$ is bounded above: $p_s(x) \leq p_{\mathrm{max}} < \infty$. Suppose $\|u_0 - \ustar\|^2_{L^2(p_s/\rho)} < \infty$. Then for $p \geq 1/2$, the updates have uniformly bounded norms:
  \begin{align}
      \|u_{k+1} - u_k\|^2_{L^2(p_s/\rho)} \leq M < \infty.
  \end{align}
\end{proposition}

\begin{proof}
    Note that
    \begin{align}
        \|u_{k+1} - u_k\|^2_{L^2(p_s/\rho)} &= \left\|\eta \rho \frac{r_k}{(\|r_k\|^2_{L^2(\rho)} + c)^p} \right\|^2_{L^2(p_s/\rho)} \\
        &= \eta^2 \int_{\mathbb R^d \times \Delta} \frac{p_s}{\rho} \rho^2  \frac{\|r_k\|^2}{(\|r_k\|^2_{L^2(\rho)} + c)^{2p}} \diff x \diff s \diff t \\
        &\leq \eta^2 \frac{p_{\mathrm{max}}}{(\|r_k\|^2_{L^2(\rho)} + c)^{2p}} \int_{\mathbb R^d \times \Delta} \rho \|r_k\|^2 \diff x \diff s \diff t \\
        &= \eta^2 p_{\mathrm{max}} \frac{\|r_k\|^2_{L^2(\rho)}}{(\|r_k\|^2_{L^2(\rho)} + c)^{2p}}.
    \end{align}
    The function $f(x) = \frac{x}{(x + c)^{2p}}$ on $x \geq 0$ is bounded for $p \geq 1/2$, since it is continuous with $f(0) = 0$ and $f(x) \leq 1$ for $x \geq 1$.
\end{proof}


\begin{proposition}
  Assume \Cref{ass_regularity_vel,ass_regularity_mf} and the discrete equivalent of \Cref{ass_integrability} hold. Consider the update scheme given in \Cref{eq:normalized_update}, with $p = 1.0$. Assume that on a subset $K^{+} = \{(\Phi^\star(x, a, s), s, t) : x \in K, a \leq s \leq t \leq b\}$, we have that $p_s$ is bounded above: $p_s(x) \leq p_{\mathrm{max}} < \infty$, and $\rho$ is bounded below: $0 < \rho_{\mathrm{min}} \leq \rho$. Then we obtain
  \begin{align}
      \|u_{k+1} - \ustar\|^2_{L^2(p_s/\rho)} \leq \max(\|u_{k} - \ustar\|^2_{L^2(p_s/\rho)}, \eta p_{\mathrm{max}} (\rho_{\mathrm{min}}^{-1} + \eta c^{-1})).
  \end{align}
    Under the assumption that there exists $M$ for which $\|r_k\|^2_{L^2(\rho)} \leq M(k+1)$ for each $k$, for all $k$ large enough we obtain $\|u_{k} - \ustar\|^2_{L^2(p_s/\rho)} \leq \eta p_{\mathrm{max}} (\rho_{\mathrm{min}}^{-1} + \eta c^{-1})$.
\end{proposition}

\begin{proof}
  We consider all norms over $K^+$. Note that
  \begin{align}
    e_{k+1} = e_{k} - \eta \rho \frac{r_{k}}{\|r_{k}\|^2_{L^2(\rho)} + c},
  \end{align}
  so with $\mathbf x = (x, s, t)$, we have that
  \begin{align}
    \int \frac{p_s}{\rho} \|e_{k+1}\|^2 \diff \mathbf{x} &= \int \frac{p_s}{\rho} \|e_{k}\|^2 \diff \mathbf{x} - 2 \eta \int \frac{p_s}{\|r_{k}\|^2_{L^2(\rho)} + c} e_{k} \cdot r_{k} \diff \mathbf{x} \\
    & \qquad + \eta^2 \int p_s \rho \frac{\|r_k\|^2}{(\|r_k\|^2_{L^2(\rho)} + c)^2} \diff \mathbf{x}.
  \end{align}
  Under the assumption that $p_s \leq p_{\mathrm{max}}$, the final term can be bounded as $\eta^2 p_{\mathrm{max}} / (\|r_k\|^2_{L^2(\rho)} + c)$. For the former, we can use a change of variables, the expansion $r_k = u_k - v - (t - s) \mathcal D_v u_k = e_k - (t - s) \mathcal D_v e_k$, and integration by parts to notice that
  \begin{align*}
    \int_{K^{+}} p_s e_{k} \cdot (e_{k} - (t - s) \mathcal{D}_v e_{k}) \diff \mathbf{x} &= \int p_s \|e_{k}\|^2 \diff \mathbf{x} - \frac{1}{2} \int (t - s) p_a \frac{\diff}{\diff s} \|e_{k}\|^2 \diff \mathbf{x} \\
    &= \frac{1}{2} \int p_s \|e_{k}\|^2 \diff \mathbf{x} + \frac{1}{2} \int (t - a) p_a \|e_{k}\|^2 \diff x \diff t.
  \end{align*}
  Hence we get the inequality
  \begin{align*}
    \|e_{k+1}\|^2_{L^2(p_s/\rho)}- \|e_k\|^2_{L^2(p_s/\rho)} &\leq  - \frac{\eta}{\|r_{k}\|^2_{L^2(\rho)} + c} \left( \int p_s \|e_{k}\|^2 \diff \mathbf{x} - \eta p_{\mathrm{max}} \right) \\
    &\leq - \frac{\eta}{\|r_{k}\|^2_{L^2(\rho)} + c} \left( \rho_{\mathrm{min}} \int \frac{p_s}{\rho} \|e_{k}\|^2 \diff \mathbf{x} - \eta p_{\mathrm{max}} \right) \\
    &\leq - \frac{\eta}{\|r_{k}\|^2_{L^2(\rho)} + c} \left( \rho_{\mathrm{min}} \|e_k\|^2_{L^2(p_s/\rho)} - \eta p_{\mathrm{max}} \right).
  \end{align*}
  In particular, if $\|e_k\|^2_{L^2(p_s/\rho)} \geq \eta \frac{p_{\mathrm{max}}}{\rho_{\mathrm{min}}}$, then $\|e_{k+1}\|^2_{L^2(p_s/\rho)} \leq \|e_k\|^2_{L^2(p_s/\rho)}$. If $\|e_k\|^2_{L^2(p_s/\rho)} \leq \eta \frac{p_{\mathrm{max}}}{\rho_{\mathrm{min}}}$, then
  \begin{align*}
    \|e_{k+1}\|^2_{L^2(p_s/\rho)} \leq \|e_k\|^2_{L^2(p_s/\rho)} + \frac{\eta^2 p_{\mathrm{max}}}{\|r_{k}\|^2_{L^2(\rho)} + c} \leq \eta p_{\mathrm{max}} \left( \frac{1}{\rho_{\mathrm{min}}} + \frac{\eta}{c} \right).
  \end{align*}
  Therefore, once we enter the ball with squared radius $\eta p_{\mathrm{max}} (\rho_{\mathrm{min}}^{-1} + \eta c^{-1}) = \mathcal{O}(\eta)$, we always stay within it.

  Let $\|e_k\|^2_{L^2(p_s/\rho)} = T_{k} + \eta p_{\mathrm{max}} \rho_{\mathrm{min}}^{-1}$. Then,
  \begin{align*}
    T_{k+1} \leq T_{k} \left( 1 - \frac{\eta \rho_{\mathrm{min}}}{\|r_{k}\|^2_{L^2(\rho)} + c} \right).
  \end{align*}
  Suppose $T_k \geq 0$ for all $k$, otherwise we have already entered the energy ball. Then, we must have $1 - \frac{\eta \rho_{\mathrm{min}}}{\|r_{k}\|^2_{L^2(\rho)} + c} \geq 0$, and so through telescoping, we can derive the inequality
  \begin{align*}
    T_{k+1} &\leq T_0 \prod_{i = 0}^{k} \left( 1 - \frac{\eta \rho_{\mathrm{min}}}{\|r_{i}\|^2_{L^2(\rho)} + c} \right) \leq T_{0} \prod_{i = 0}^{k} \exp \left( -  \frac{\eta \rho_{\mathrm{min}}}{\|r_{i}\|^2_{L^2(\rho)} + c} \right) \\
    &= T_{0} \exp \left( - \eta \rho_{\mathrm{min}} \sum_{i = 0}^{k} \frac{1}{\|r_{i}\|^2_{L^2(\rho)} + c} \right).
  \end{align*}
  Hence under the assumption
  \begin{align*}
    \sum_{i = 0}^{\infty} \frac{1}{\|r_{i}\|^2_{L^2(\rho)} + c} = \infty,
  \end{align*}
  we will have that $T_{k} \to 0$ as $k \to \infty$. Hence the squared norm is confined to a ball of squared radius $\mathcal{O}(\eta)$. Finally, this condition holds true if $\|r_k\|^2_{L^2(\rho)} \leq M (k+1)$ for all $k$.
\end{proof}

This says that we converge under the assumption that the \emph{effective total step size}
  \begin{align*}
    \sum_{i = 0}^{\infty} \frac{1}{\|r_{i}\|^2_{L^2(\rho)} + c}
  \end{align*}
is infinite.

For $p > 1$ it is likely possible to formulate similar convergence guarantees; however the assumptions on the residual norms will be stricter.

\subsection{Equivalence to Align Your Flow}
\label{sub:ayf_equiv}

We show that \Cref{eq:mf_loss,eq:lagrangian_loss} have the same gradient update as the losses in Align Your Flow \citep{sabourAlignYourFlow2025}. This is already proven for the AYF-EMD objective and the \meanflow loss in Section D.5 of \citet{sabourAlignYourFlow2025}.

For the AYF-LMD objective, note that
\begin{align}
\nabla_\theta \mathcal L^{\mathrm{AYF\text{-}LMD}
}
=
\nabla_\theta
\mathbb E\left[
w(s,t)(t-s)
f_\theta(x,s,t)^\top
\mathrm{sg}\left(
\partial_t f_\theta(x,s,t)
-v(f_\theta(x,s,t),t)
\right)
\right].
\end{align}
Using the parametrization
\begin{align*}
f_\theta(x,s,t)=x+(t-s)u_\theta(x,s,t),
\end{align*}
we have
\begin{align*}
\nabla_\theta f_\theta
&=(t-s)\nabla_\theta u_\theta,\\
\partial_t f_\theta-v(f_\theta,t)
&=
u_\theta
+(t-s)\partial_tu_\theta
-v\left(x+(t-s)u_\theta,t\right).
\end{align*}
Therefore,
\begin{align}
\nabla_\theta \mathcal L^{\mathrm{AYF\text{-}LMD}}
=
\mathbb E\left[
w(s,t)(t-s)^2
\nabla_\theta u_\theta^\top
\mathrm{sg}\left(
u_\theta
+(t-s)\partial_tu_\theta
-v\left(x+(t-s)u_\theta,t\right)
\right)
\right].
\end{align}
The term inside the \stopgradient is precisely the residual of \Cref{eq:lagrangian_loss}. Hence AYF-LMD and our Lagrangian loss induce the same gradient update, up to their choice of loss weightings.

\clearpage
\section{Experimental details}
\label{sec:exp_details}

\subsection{Additional experiments: validation of theoretical results}
\label{sub:validation_appendix}

\begin{figure}
    \centering
    \includegraphics[width=0.6\linewidth]{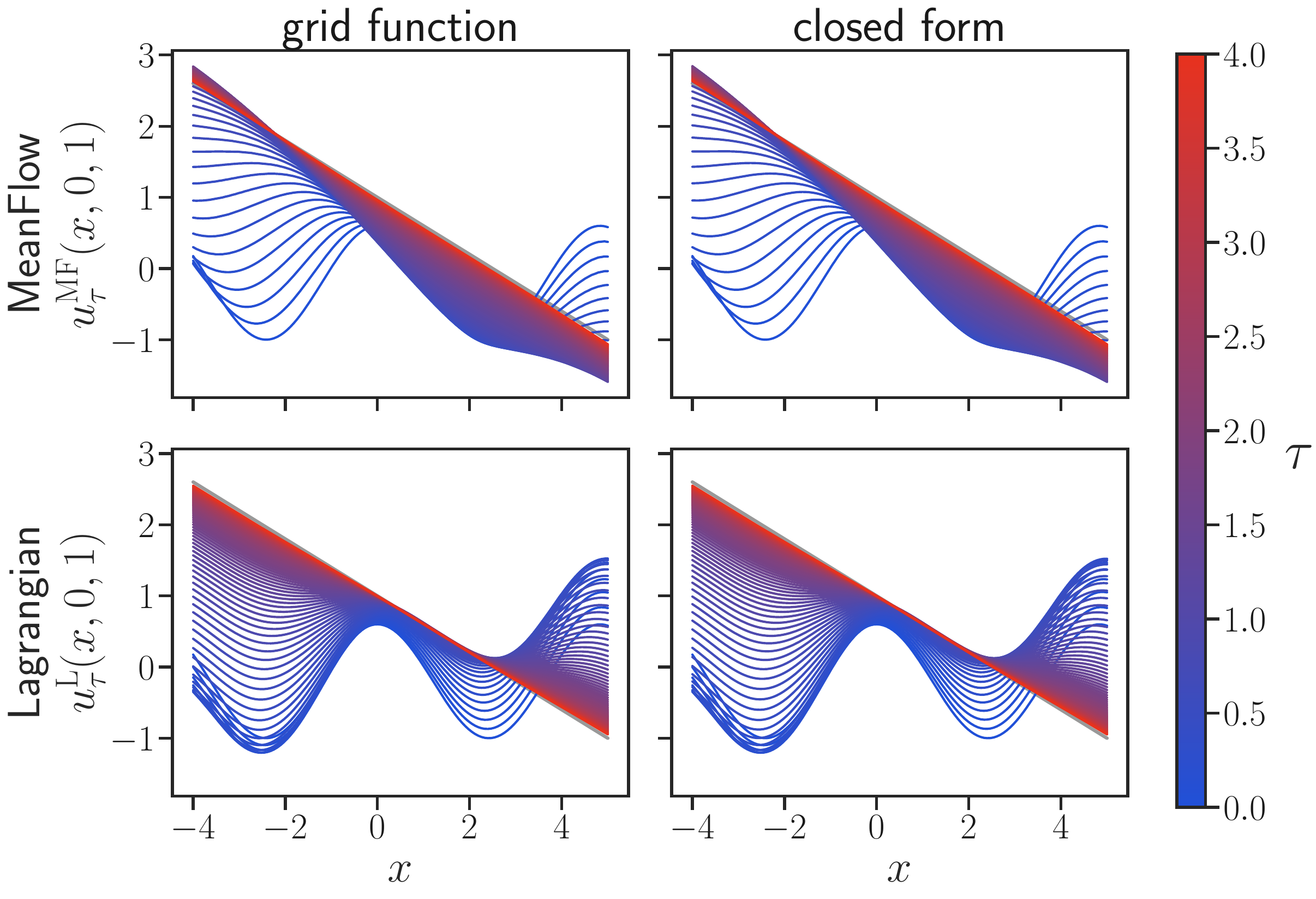}
    \caption{
    \emph{Synthetic Gaussian-to-Gaussian dataset.}
    Each line corresponds to an optimization time $\tau$ for the empirically learned average velocity (left) and the derived closed-form average velocity (right).
    As the optimization time increases, the average velocities converge towards the true one $u^\star$.
    The two average velocities, empirically learned and the closed-form, show strong agreement across optimization times $\tau$ and for both the \meanflow (top) and Lagrangian target (bottom).
    }
    \label{fig_additional_valid_theo}
\end{figure}

All experiments in this subsection were conducted on CPU.

\xhdr{Equivalence of experimental and theoretical trajectories} \Cref{fig:equals_closed_form_1d,fig_additional_valid_theo} show $1 + 2$ dimensional lookup tables with $129 \times 33 \times 33$ parameters trained from a cosine initialization relaxing towards the true flow map. We choose a bounded region of space $x \in [-4, 5]$, and initialize $u_0(x, s, t) = 0.8 \cos (1.3x + 2s) - 0.2t$. Choosing $p_0 = \mathcal{N}(0, 1)$, $p_1 = \mathcal{N}(1, 0.6^2)$, the linear interpolant has $p_t = \mathcal{N}(t, (1 - t)^2 + t^2 0.6^2)$. We take $\rho = 1$, and use learning rate $3 \times 10^{-5}$ for both methods.



\xhdr{Effect of loss weighting on convergence} In order to see how loss weighting affects convergence, we redo the same convergence experiments with $\rho(x, s, t) = (t - s)^{a}$, for various values of $a$. In order for the updates to remain stable at $a = -1$, we use a coarser grid, of size $65 \times 17 \times 17$. Results are plotted in \Cref{fig:different_loss_weightings}. We choose an LR of $10^{-5}$ for $a = -1$, $3 \times 10^{-5}$ for $a = -0.5$, and $3 \times 10^{-4}$ for $a = 0$ and $1$.

We can see that our derived convergence rates hold in practice. Moreover, for choices of $a < 0$ when interior characteristics are predicted to hit the boundary at finite time, convergence to the error floor happens at or shortly after this hitting time.

\begin{figure}[h]
  \centering
  \includegraphics[width=\textwidth]{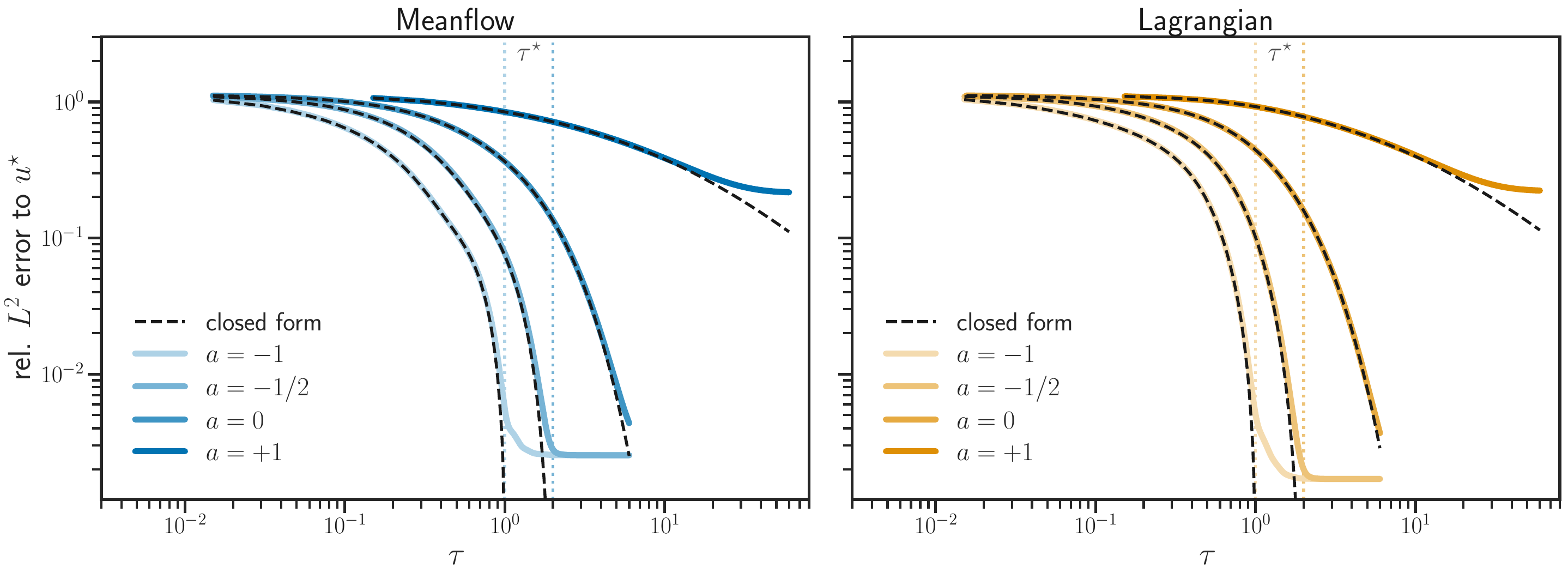}
  \caption{A comparison of the convergence rates of flow maps trained with different loss weightings $\rho = (t - s)^{a}$, versus their theorized trajectories. For $a = -1$ and $a = -0.5$, the theoretical last diagonal hitting time is marked.}
  \label{fig:different_loss_weightings}
\end{figure}

To quantify the dependence of each loss weighting on the diagonal, we consider a situation where the diagonal is not $\tau$-independent. Using the same experimental setup as with \Cref{fig:different_loss_weightings}, we now let $u_\tau(x, t, t) = v(x, t) + \sin (\pi \tau)/2$. 

\Cref{fig:oscillating_diagonal_weighting} shows a comparison of the $L^2$ errors for $a \leq 0$. We see that for $a = -0.5$ and $a = -1$, the loss is consistently higher than the $a = 0$ loss, which matches our prediction that the interior of the PDE tracks the boundary for $a < 0$, whereas for $a = 0$ the solution of the PDE in the interior should be independent of the evolution of the boundary. The failure of the $a = 0$ runs to recover the true interior map is a discretization error that arises from our lookup table approximation.

\begin{figure}[h]
  \centering
  \includegraphics[width=\textwidth]{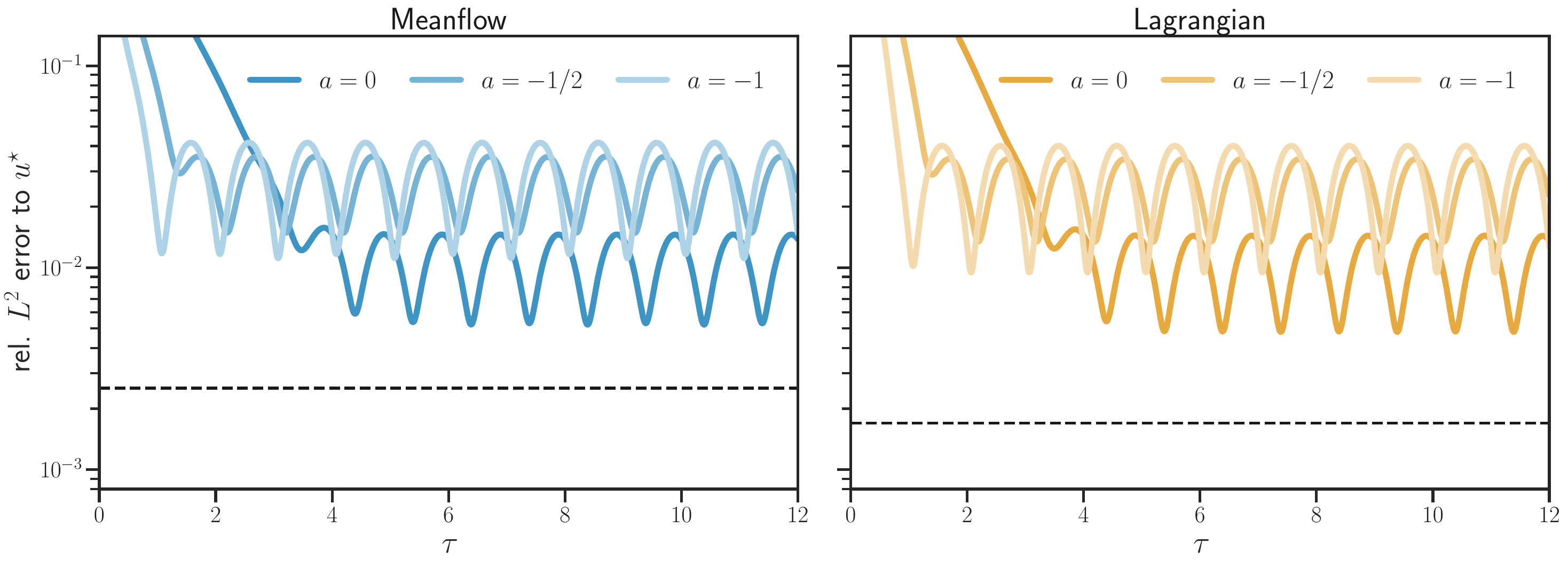}
  \caption{Instead of keeping the diagonal $u(x, s, s)$ fixed, we vary it with optimization time, keeping it centered at the true velocity. The learned flow maps for $a < 0$ have a higher MSE against the true flow map, compared to $a = 0$. We plot the errors achieved when the diagonal is fixed for comparison.}
  \label{fig:oscillating_diagonal_weighting}
\end{figure}





\xhdr{Mode-learning dependence on $\eta$} We now show that learning higher modes is impossible with a fixed learning rate. Using the linear interpolant between $\mathcal{N}(0, 1)$ and $\mathcal{N}(0, 0.6^2)$, we consider initializing our map with error $(t - s)^{k}$, for choices of $k = 2$, $8$ and $24$. We evaluate the largest error of the \meanflow evolution on a $64 \times 257 \times 257$ grid for different values of $\eta$.

Results are found in \Cref{fig:eta_vs_mode}. We see that if $\eta$ is too large compared to $1/k$, the table is unable to learn, and the error grows.

\begin{figure}[h]
  \centering
  \includegraphics[width=\textwidth]{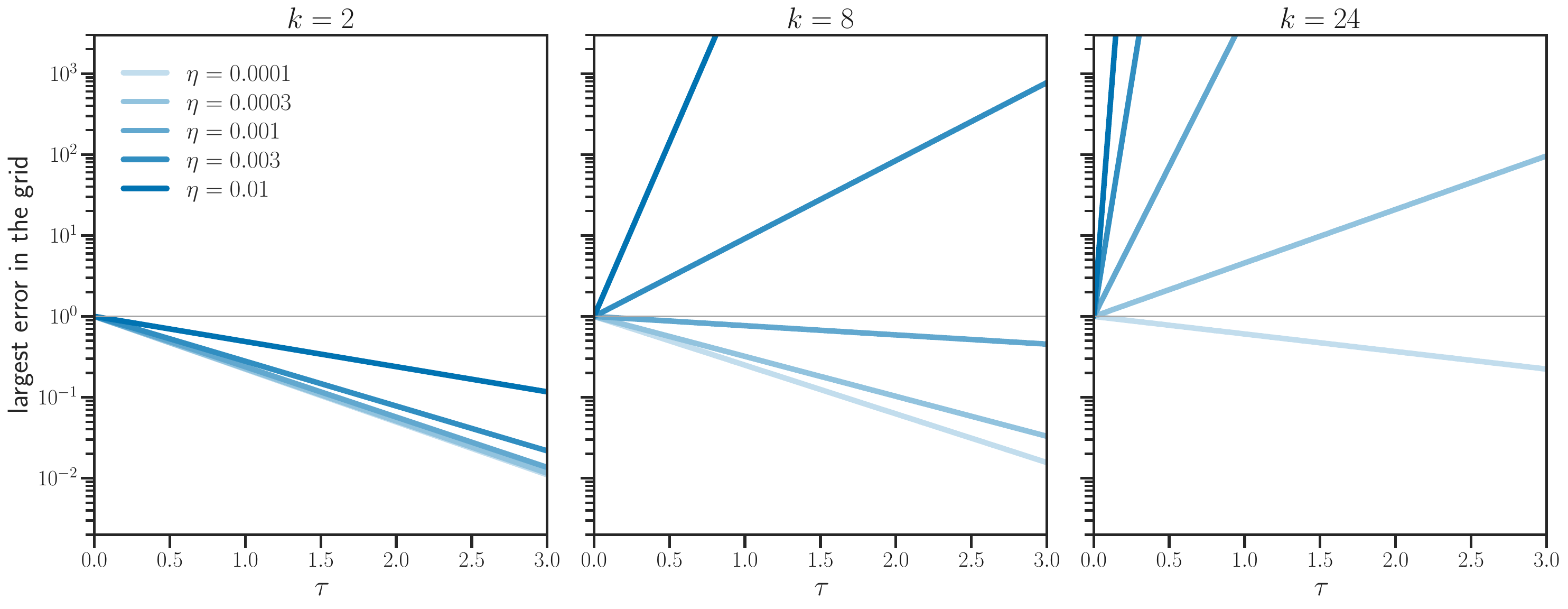}
  \caption{The choice of $\eta$ impacts whether we learn an error of the form $(t - s)^{k}$. As $k$ increases, a smaller $\eta$ must be taken to ensure that gradient descent does not accumulate errors.}
  \label{fig:eta_vs_mode}
\end{figure}

\subsection{Additional experiments: qualitative behavior}




We aim to show a difference in practice between learning using the \meanflow and the Lagrangian objective. The following experiment ran in 90 minutes on a single RTX 5090.

\noindent\textbf{Dependence on velocity Lipschitzness.}
We have shown that pointwise and $L^2$ convergence rates for the \meanflow objective are independent of the Lipschitz constant of the velocity $v$, whereas in the Lagrangian objective the convergence guarantees are exponentially weakened by the Lipschitz constant $L$.
To test whether this effect is visible in practice, we train neural flow map models to follow increasingly expansive velocities using both objectives, and measure their relative $L^2$ error to the true average velocity $\ustar$.

\begin{figure}[tb]
    \centering
    \vspace{-1.5ex}
    \includegraphics[width=.75\linewidth]{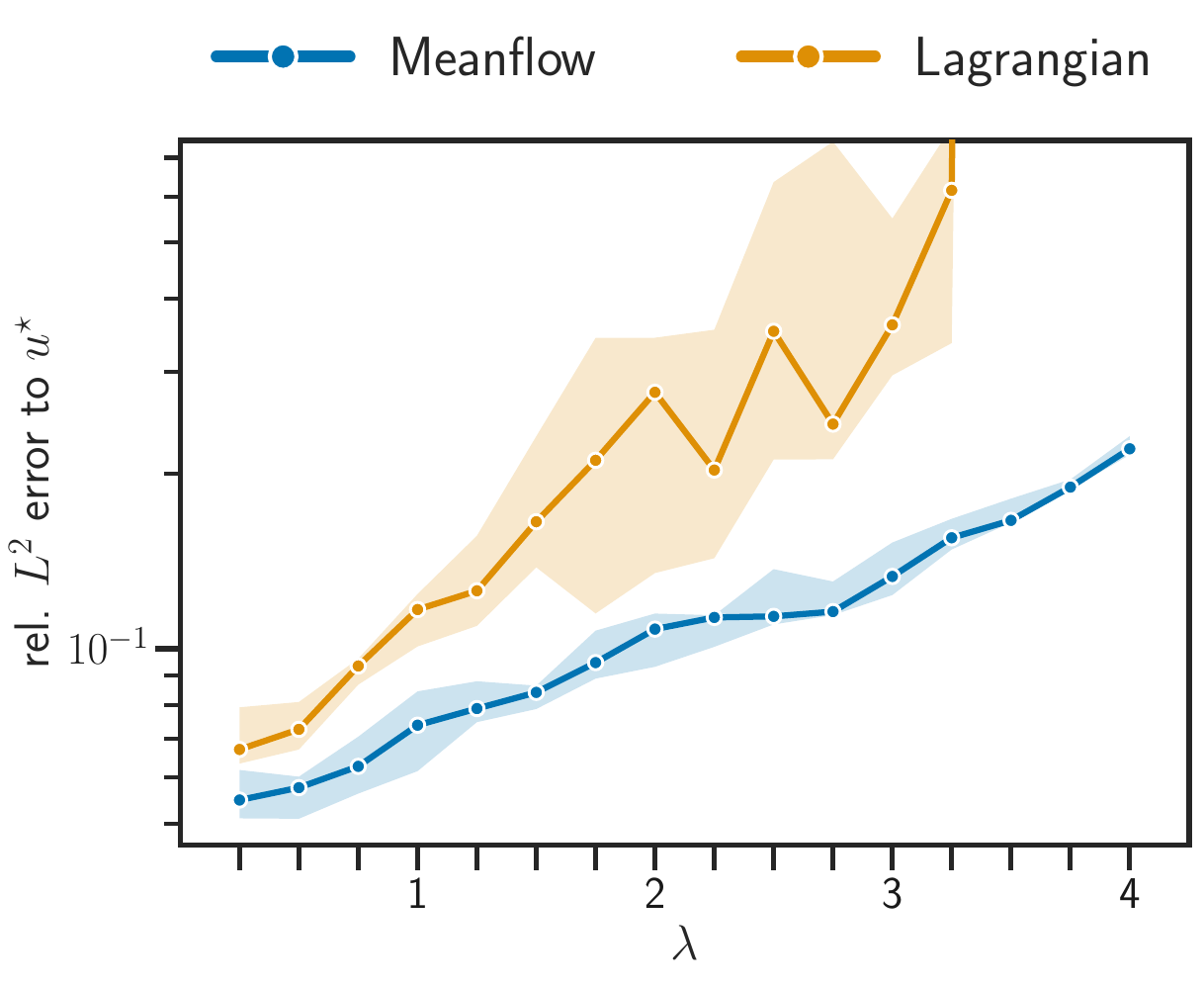}
    \caption{For the linear flow $v(x, t) = \lambda x$, we report the average relative $L^2(p_s)$ error over 5 seeds for average velocity networks against the ground truth $\ustar = (e^{\lambda(t-s)} - 1)x/(t - s)$. Lagrangian is worse than \meanflow for every $\lambda$ tested, and is unstable at high $\lambda$.}
    \label{fig:expansive_mlp_relative_ls}
\end{figure}

We train $3 \times 256$ MLPs with Adam with learning rate $10^{-4}$ to learn the flow map associated to the velocity $v(x, t) = \lambda x$ for values of $\lambda$ in $[0.25, 4]$, starting from the 64 dimensional Gaussian distribution $\mathcal N(0, I)$. We train for 60k steps, using batch sizes of 1024. We plot the mean relative $L^2$ error obtained over 5 seeds along with min-max bands in \Cref{fig:expansive_mlp_relative_ls}.
The \meanflow objective achieves lower error on this task than the Lagrangian objective over all values of $\lambda$ considered, which is consistent with the convergence rate theory that we have developed. For values of $\lambda$ above 3.5, the Lagrangian objective fails to learn the flow map.










To keep the input at bounded variance, we scale the network input by $e^{-\lambda}$, and the network output by $e^{\lambda}$. We compute $L^2$ errors over 30k samples, sampled from $p_s(x)$ for $s \sim U[0,1]$.

\clearpage
\section{Additional experimental details for large scale experiments}
\label{app_expe_large_scale_details}

We include the performance of Lagrangian-trained flow maps on CIFAR-10 in \Cref{tab:fid_p_ablation_ilsd_ffhq}. Similarly to the \meanflow trained flow maps, the lowest FIDs were produced when $p \geq 0.5$, and $p = 0.0$ gave the worst performance.

We provide the hyperparameters used for training in \Cref{tab:training_recipe}. Note that we ablate over the $p$ parameter that decides the strength of the adaptive normalization $w = \mathrm{sg}( (\|u_k-u_{\mathrm{tgt},k}\|_2^2 + c)^{-p})$, with $c = 10^{-3}$. We either use the \meanflow or the Lagrangian loss with the OT-CFM teacher \citep{tong2024improving} for CIFAR-10, or the EDM teacher \citep{EDM} for FFHQ 64$\times$64. Times are measured for the \meanflow loss.
Each entry of \Cref{tab:fid_p_ablation,tab:fid_p_ablation_ilsd_ffhq} displays the mean and the standard deviation over 3 training runs.

Following \citet{gengMeanFlowsOnestep2025}, we use $(s, t - s)$ as the time embedding rather than $(s, t)$, and we pick $(s, t)$ from specialized distributions. For off-diagonal points, we use $s \sim \mathrm{LogitNormal}(0.6, 1.6^2)$ and $t \sim \mathrm{LogitNormal}(4.0, 1.6^2)$, then set $(s, t) \leftarrow (\min(s, t), \max(s, t))$. For diagonal points, $s \sim \mathrm{LogitNormal}(0.6, 1.6^2)$.

\begin{table}[t]
\centering
\caption{
FID on CIFAR-10 for the Lagrangian objective with different normalization exponents $p$.
Results are reported as mean $\pm$ standard deviation over three runs.
}
\label{tab:fid_p_ablation_ilsd_ffhq}
\small
\setlength{\tabcolsep}{8pt}
\renewcommand{\arraystretch}{1.05}
\begin{tabular}{cccc}
\toprule
$p$ & NFE 1 & NFE 2 & NFE 4 \\
\midrule
0.00 & $74.01 \pm 8.71$ & $25.82 \pm 3.51$ & $9.75 \pm 1.16$ \\
0.25 & $56.61 \pm 1.46$ & $21.51 \pm 0.85$ & $8.14 \pm 0.71$ \\
0.50 & $51.99 \pm 8.56$ & $19.24 \pm 1.30$ & $6.80 \pm 0.15$ \\
0.75 & $42.54 \pm 1.87$ & $15.25 \pm 1.06$ & $\mathbf{5.94 \pm 0.24}$ \\
1.00 & $\mathbf{39.58 \pm 5.96}$ & $\mathbf{13.94 \pm 2.17}$ & $6.35 \pm 0.67$ \\
1.25 & $54.33 \pm 7.59$ & $21.00 \pm 3.00$ & $8.06 \pm 1.17$ \\
1.50 & $45.85 \pm 2.94$ & $20.32 \pm 1.93$ & $8.47 \pm 1.05$ \\
\bottomrule
\end{tabular}
\end{table}

\begin{table*}[t]
\centering
\caption{Training configuration for the CIFAR-10 and FFHQ experiments.}
\label{tab:training_recipe}
\small
\renewcommand{\arraystretch}{1.15}
\begin{tabular}{lcc}
\hline
 & \textbf{CIFAR distilled}
 & \textbf{FFHQ distilled} \\
\hline
Teacher
    & CFM (EMA)
    & EDM VE (EMA) \\
    Architecture
    & EDM2-style U-Net
    & EDM2-style U-Net \\
Parameters
    & 54.7M
    & 59.9M \\
Model Channels 
    & 128
    & 128 \\
Channel Mults
    & [2, 2, 2]
    & [1, 2, 2, 2] \\
Attention Dims
    & 16, 8 
    & 16, 8 \\
Training time (1$\times$H100)
    & 13.5h
    & 28.3h \\
Seeds 
    & 0, 1, 2
    & 0, 1, 2 \\
Diagonal Proportion
    & 0.5
    & 0.5 \\
Optimizer
    & Adam
    & Adam \\
Learning rate
    & $2\times10^{-3}$
    & $1\times10^{-3}$ \\
Dropout
    & 0.2
    & 0.2 \\
Batch size
    & 256
    & 256 \\
Training steps
    & 200k
    & 200k \\
Warm-up
    & 2.5k
    & 2.5k \\
EMA
    & 0.9999
    & 0.9999 \\
Precision
    & bf16 AMP
    & bf16 AMP \\
\hline
\end{tabular}
\end{table*}






\end{document}